\documentclass[conference]{IEEEtran}
\IEEEoverridecommandlockouts
\usepackage{cite}
\usepackage{amsmath,amssymb,amsfonts}
\usepackage{algorithmic}
\usepackage{graphicx}
\usepackage{textcomp}
\usepackage{xcolor}

\usepackage{amsthm}
\usepackage{multirow}
\usepackage{threeparttable}
\newtheorem{theorem}{Theorem}[section]
\newtheorem{lemma}[theorem]{Lemma}
\newtheorem{proposition}[theorem]{Proposition}
\newtheorem{definition}{Definition}
\usepackage{framed}
\newcommand{\hlr}{\textcolor{black}}

\usepackage{tikz}
\newcommand*\emptycirc[1][1ex]{\tikz\draw[thick] (0,0) circle (#1);} 

\newcommand*\fullcirc[1][1ex]{\tikz\fill (0,0) circle (#1);}

\newcommand{\ec}{\emptycirc[0.8ex]}

\newcommand{\fc}{\fullcirc[0.9ex]}

\usepackage{hyperref}
\hypersetup{
	colorlinks=true,
	linkcolor=red,
	filecolor=magenta,      
	urlcolor=blue,
	citecolor=black
}
\usepackage{soul}
\usepackage{balance}

\usepackage[absolute]{textpos}

\def\BibTeX{{\rm B\kern-.05em{\sc i\kern-.025em b}\kern-.08em
    T\kern-.1667em\lower.7ex\hbox{E}\kern-.125emX}}
\begin{document}

\begin{textblock*}{10cm}(1.3cm,1cm) 
    Accepted by IEEE ICDE 2024.
\end{textblock*}

\title{Unraveling Privacy Risks of Individual Fairness in Graph Neural Networks\\
}

\author{\IEEEauthorblockN{He Zhang}
\IEEEauthorblockA{\textit{Faculty of IT} \\
\textit{Monash University}\\
Melbourne, Australia \\
he.zhang1@monash.edu}
\and
\IEEEauthorblockN{Xingliang Yuan}
\IEEEauthorblockA{\textit{School of CIS} \\
\textit{The University of Melbourne}\\
Melbourne, Australia \\
xingliang.yuan@unimelb.edu.au}
\and
\IEEEauthorblockN{Shirui Pan}
\IEEEauthorblockA{\textit{School of ICT} \\
\textit{Griffith University}\\
Gold Coast, Australia \\
s.pan@griffith.edu.au}
}


\maketitle

\begin{abstract}
Graph neural networks (GNNs) have gained significant attraction due to their expansive real-world applications. To build trustworthy GNNs, two aspects - fairness and privacy - have emerged as critical considerations. Previous studies have separately examined the fairness and privacy aspects of GNNs, revealing their trade-off with GNN performance. Yet, the interplay between these two aspects remains unexplored. In this paper, we pioneer the exploration of the interaction between the privacy risks of edge leakage and the individual fairness of a GNN. Our theoretical analysis unravels that edge privacy risks unfortunately escalate when the nodes' individual fairness improves. Such an issue hinders the accomplishment of privacy and fairness of GNNs at the same time. To balance fairness and privacy, we carefully introduce fairness-aware loss reweighting based on influence function and privacy-aware graph structure perturbation modules within a fine-tuning mechanism. Experimental results underscore the effectiveness of our approach in achieving GNN fairness with limited performance compromise and controlled privacy risks. This work contributes to the comprehensively developing trustworthy GNNs by simultaneously addressing both fairness and privacy aspects. 
\end{abstract}


\section{Introduction}
\label{sec:intro}

In recent years, in addition to competent performance \cite{ZhengZLZWP23, abs-2401-06176, zheng2022multi, zheng2023auto, luollm_kg2024, pan2024integrating, abs-2310-01728}, there has been an increasing desire for fairness \cite{ChiragHM2021,LiWZHL21, WangJ23, FabrisSSB23} and private information security \cite{WuYPY22,WuYPY21, Gondara0C22, DZhaoZ00L23} in GNNs, because both privacy \cite{PanZC23, abs-2312-07861, zheng2023structure, abs-2312-07870} and fairness \cite{YinZZLW23,MaGWSW23, JinWZZDXP23} are two essential concerns of GNN users \cite{abs-2205-07424}. 
In the context of GNNs \cite{ZhangYZP2022, ZhangWY0WYP21, zheng2023towards, zheng2023gnnevaluator}, existing studies \cite{HeJ0G021,KangHMT20} focus only on performance and privacy/fairness; however, privacy and fairness are not isolated from each other in practical scenarios. 
As shown in Fig. \ref{fig:background}, in a graph data system \cite{LiangZ17, MassriMPM22} for efficient retrieval, a well-trained GNN is used to predict the retrieval probability of nodes in a graph database for a given datum node. 
In addition to the prediction accuracy, fairness requires that similar node candidates obtain similar probabilities in the final prediction results. However, the confidential edges in the graph database are potentially facing more serious privacy risks when improving fairness in the GNN-based retrieval system.
Therefore, it is necessary to study the interaction among performance, fairness, and privacy to build trustworthy GNNs comprehensively \cite{abs-2205-07424,abs-2212-04481, luo2024rog, abs-2310-07984, LiuD0LZP23}.

\begin{figure}
  \centering
  \includegraphics[width=\linewidth]{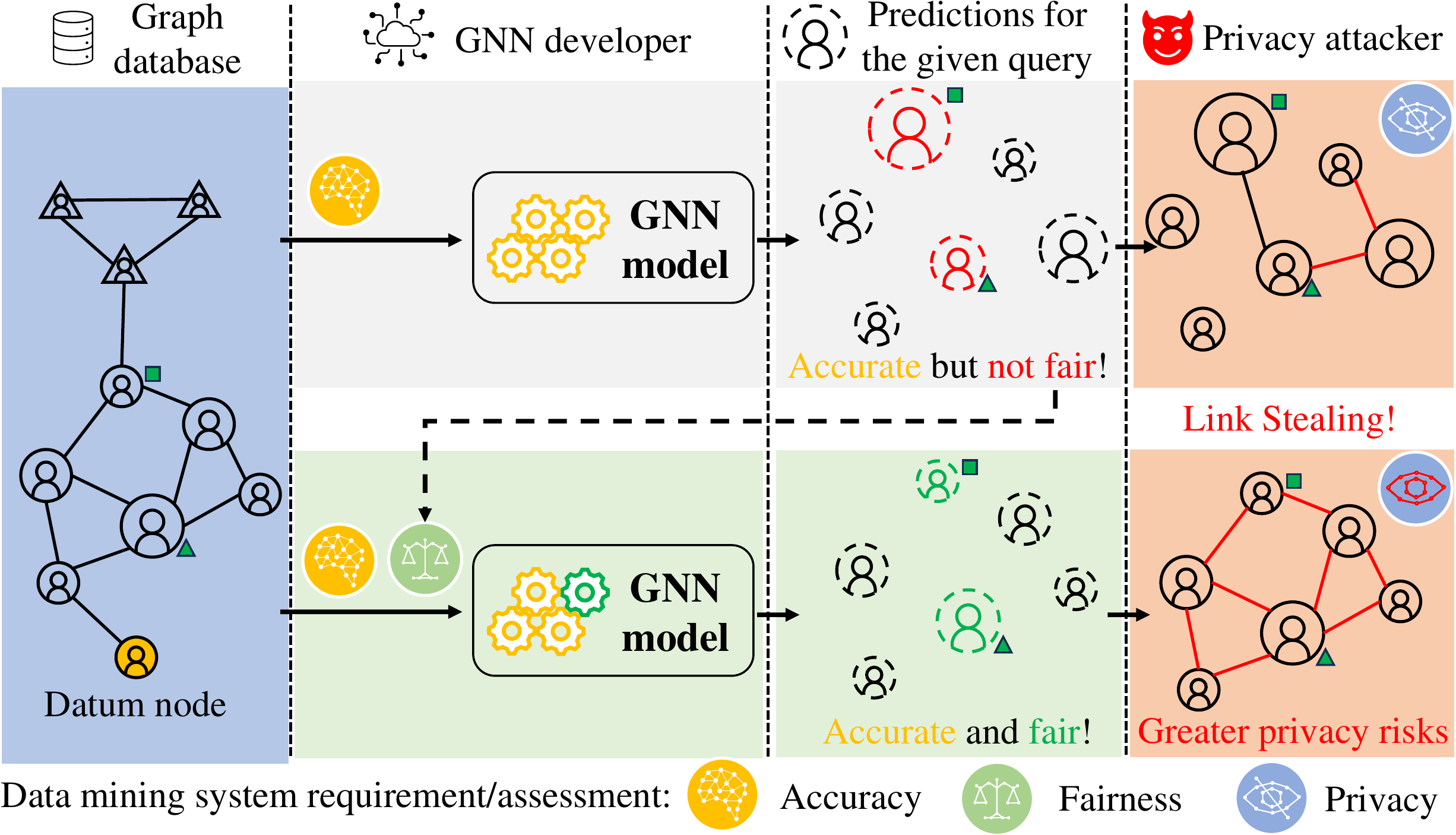}
  \vspace{-20pt}
  \caption{
  A GNN-based graph database querying system. Given the datum node in a query, the size of other nodes indicates the node similarity score for the given datum. After the vanilla training (i.e., the grey zone),  although all other nodes have been accurately predicted by a GNN from the topological view, complaints are raised as the absence of prediction fairness (e.g., for the triangle-indicated node, its size in the prediction space should be larger than the square-indicated node as that in the graph database). To this end, model developers should train GNNs by considering both accuracy and fairness (i.e., the lower green zone). However, privacy attackers can more easily infer confidential edges if GNN fairness is improved.
  }
  \label{fig:background}
\vspace{-10pt}
\end{figure}

The existing literature has studied the interaction between performance and fairness/privacy of GNNs. 
For example, to increase individual fairness in GNNs, a method called REDRESS \cite{DongKTL21} proposes to add a regularisation for fairness from the ranking perspective in the loss function of GNNs. 
Moreover, with the Lipschitz property, InFoRM \cite{KangHMT20} proposes a metric to evaluate individual fairness of nodes, which can also be involved in the loss function of a target model to reduce the bias existing in GNN predictions.
For edge privacy risks, Wu \textit{et al.} \cite{0011L0022} explore the vulnerability of edges in the training graph and use differential privacy (DP) mechanisms to prevent edges from leakage. 
Note that both lines of those studies demonstrate that achieving individual fairness or edge privacy comes at the cost of GNN performance. 


In the quest to comprehensively construct trustworthy GNNs, a critical aspect that necessitates exploration is the interplay between fairness and privacy. A limited number of studies have ventured into the realm of understanding the implications of improving model privacy on fairness \cite{zhang2021balancing}, or the reciprocal effects of promoting algorithmic fairness on privacy \cite{ChangS21}. These investigations, however, predominantly pertain to the domain of general machine learning models for independent identically distributed (IID) data. 
In contrast, this paper embarks on an unexplored path, investigating the interaction between the individual fairness of nodes and the privacy risks associated with edges in GNNs, specifically within the context of complex graph data. 

To the best of our knowledge, this is a new area of GNN study \cite{WangGB0023}. 
Building upon this, devising a method that simultaneously upholds privacy, fairness, and competent performance of GNNs is confronted with two significant challenges that emerge from the current state of research. 
\textbf{(1)} The first challenge lies in navigating the potential trade-off between fairness and privacy (e.g., ${model\ bias\downarrow \ \Rightarrow\ privacy\ risk \uparrow}$). 
\textbf{(2)} The second challenge arises from the fact that the pursuit of fairness or privacy often compromises performance 
(i.e., ${bias\downarrow \ \Rightarrow accuracy \downarrow}$ and ${risk \downarrow \ \Rightarrow accuracy \downarrow}$). 

To address those challenges, we first theoretically and empirically verify the adverse effect of individual fairness of nodes on edges' privacy risks in GNN models.
Furthermore, we propose a Privacy-aware Perturbations and Fairness-aware Reweighting (PPFR) method to implement GNN fairness with limited performance cost and restricted privacy risks.
Our three-fold contributions are summarised as follows:
\begin{itemize}
    \item 
    In this paper, we explore the interaction between fairness of nodes and privacy of edges in GNNs for the first time. 
    We theoretically and empirically show that privacy risks of edges potentially increase when improving individual fairness of nodes, i.e., there is a \textit{trade-off between fairness and privacy}.
    \item With the loss reweighting based on influence function and edge perturbation driven by our theoretical analysis, we propose a novel method to obtain competent GNNs with reduced bias and edge leakage risks, 
    which is \textit{model-agnostic} and serves in a \textit{plug-and-play} manner.
    \item Experimental results demonstrate the \textit{superiority} (i.e., lower $|\Delta_{acc}|$) and \textit{effectiveness} (i.e., $bias$ \& $risk\downarrow$ simultaneously) of our method over directly combining existing GNN fairness and privacy methods, promoting the construction of trustworthy GNNs from both privacy and fairness views.
\end{itemize}

\section{Related Work}
\label{sec:relatedwork}
\subsection{Fairness}
\label{sec:relatedwork:fairness}
There has been concern raised recently about prediction discrimination in current GNN systems, and various efforts have been made to improve GNN fairness in their essential tasks such as node classification, where group fairness and individual fairness are two typical design goals.

Group fairness expects that GNNs treat different groups of users equally. 
For example, the demographic parity \cite{KusnerLRS17} aims to make the predicted label (e.g., ${y\in \{0,1\}}$ in binary node classification) of a GNN independent of the membership of nodes in a sensitive group (e.g., female). It can be defined by ${P(y \mid s=0)=P(y \mid s=1)}$, where $s$ indicates if a node is in a sensitive group \cite{VermaR18} and $P$ denotes probability.
Following this goal, a method called FairGNN \cite{DaiW21} improves the group fairness in situations where grouping attributes (e.g., regions) are sparsely annotated.
To avoid fine-tuning fairness methods (e.g., FairGNN), 
EDITS \cite{DongLJL22} proposes a model-agnostic method to mitigate bias in graph data, where the updated graph can be released to develop GNNs with group fairness.

Individual fairness requires that similar users are treated similarly by GNNs, as shown in Section \ref{sec:problem}. 
To achieve this, a method called REDRESS \cite{DongKTL21} proposes to involve a fairness-aware regularisation in the loss function during the normal GNN training. 
With the Jaccard similarity obtained from graph structure or the cosine similarity derived from node features, InFoRM \cite{KangHMT20} proposes another fairness-aware regularisation (Section \ref{sec:problem}) by enforcing the Lipschitz property \cite{DworkHPRZ12, XieMXY21} of GNNs, which is also model-agnostic and can be used in a plug-and-play manner. 
Recently, another method called GFairHint \cite{abs-2305-15622} constructs a fairness graph whose edge weight is derived from node similarity and trains an additional module to obtain the fairness hint (i.e., node embedding from the fairness graph). 
The prediction fairness is improved in downstream tasks by concatenating the fairness hint and node embedding from the backbone GNN together.

This paper focusses on the individual fairness of nodes. Specifically, we concentrate on regularisation methods \cite{KangHMT20}, as model developers with limited resources tend to choose the method with fewer implementation efforts \cite{AlsaadiS22}.
For the similarity in individual fairness, we use the Jaccard similarity derived from edges (refer to Section \ref{sec:back}) for two reasons. 
\textbf{(1)} The edge contributes to the uniqueness of graph data and makes it different from other data (e.g., tabular data).
\textbf{(2)} The graph structure is always available for graph data, even in some practical scenarios \cite{AdamicG05} where node features are not provided.
Moreover, note that methods for group fairness(e.g., FairGNN) generally cannot be used to improve individual fairness, as they have different design goals.

\subsection{Privacy}
\label{sec:relatedwork:privacy}
In addition to the GNN model, recent studies have shown that graph data is also vulnerable to privacy attacks \cite{abs-2205-07424}. 
A typical attack on nodes is the membership inference attack \cite{abs-2102-05429}, which aims to infer if a node is in the training graph of a GNN model.
Next, we will present representative attack and defence methods for the privacy of edges.

A typical edge-level privacy attack is the link stealing attack \cite{HeJ0G021}.
In its black-box attack setting, attackers can infer the existence of edges in the training graph only by querying the target GNN just once. 
This attack is motivated by ``two nodes are more likely to be linked if they share more similar attributes and/or predictions''.
Another method called LinkTeller \cite{0011L0022} obtains the connection status of a target node (e.g., $v_i$) by querying the target GNN twice and comparing the prediction difference before and after adding perturbations to the node features. 
To reduce the privacy risk, Wu \textit{et al.} \cite{0011L0022} propose the edge DP method where noises sampled from a random or Laplacian distribution are added to the original graph structure, effectively reducing the performance of privacy attack methods from the differential privacy view.  
With the same design goal, Wang \textit{et al.} \cite{WangGLCL21} present a privacy-preserving representation learning method, where the mutual information between the nodes and edges are minimised during the training of GNNs.

In this paper, we focus on link stealing attacks because it requires fewer query efforts from attackers. For the defence side, we consider the edge DP \cite{0011L0022} as it can be deployed easily and needs less cost of GNN developers, comparing with the privacy-preserving representation learning method where training two additional auxiliary networks is necessary. 

\begin{table}
\centering
\caption{Research differences between our paper and other methods.}
\label{tab:scope}
\vspace{-5pt}
\begin{threeparttable}
\resizebox{0.4\textwidth}{!}{
\begin{tabular}{c|cc|ccc|c}
\hline
\multirow{2}{*}{Methods}                             & \multicolumn{2}{c|}{Scope}                          & \multicolumn{3}{c|}{Design Goal}                                        & \multirow{2}{*}{Interplay\tnote{*}} \\ \cline{2-6}
& \multicolumn{1}{l|}{AI\tnote{*}} & \multicolumn{1}{l|}{GNNs} & \multicolumn{1}{l|}{Acc\tnote{*}} & \multicolumn{1}{c|}{Fairness} & Privacy &                            \\ \hline
\cite{XuYW19,HuLWL19,DingZLWYP20,MozannarOS20, 0007FH21} & \fc                      & \ec                         & \fc                            & G                           & \multicolumn{1}{c|}{\fc}       & \ec
\\ 
\cite{zhang2021balancing} & \fc                      & \ec                         & \fc                            & G                           & \multicolumn{1}{c|}{\fc}       & \fc
\\
\hline
\cite{KangHMT20} \cite{DongLJL22}  & \ec                       & \fc                          & \fc                              & \fc                              & \multicolumn{1}{c|}{\ec}    & \ec    \\
\cite{0011L0022} \cite{WangGLCL21}  & \ec                       & \fc                          & \fc                              & \ec                             & \multicolumn{1}{c|}{\fc}  & \ec       \\
\hline
\cite{DaiW21}       & \ec                       & \fc                          & \fc                              & G                            & \multicolumn{1}{c|}{N}  & \ec    \\
\cite{WangGB0023}      & \ec                       & \fc                          & \fc                              & I                            & \multicolumn{1}{c|}{N}  & \ec  \\
\textbf{Ours}    & \ec                       & \fc                          & \fc                              & I                            & \multicolumn{1}{c|}{E}   & \fc   \\ \hline
\end{tabular}
}
\begin{tablenotes}
    \footnotesize 
    \item[*] In this table, AI represents model design for non-graph data (e.g., tabular data). 
    ``Acc'' denotes accuracy, and ``Interplay'' concerns fairness and privacy.
    \ec/\fc  \ indicates ``Not Covered''/``Covered". ``G''/``I'' denotes group/individual fairness of nodes/samples, and ``N''/``E'' represents the privacy of nodes/edges.
\end{tablenotes}
\end{threeparttable}
\vspace{-10pt}
\end{table}

\subsection{Fairness and Privacy}
Existing works of considering fairness and privacy simultaneously can be categorised into relationship study and method study.
For the relation aspect, limited studies discuss the implications of improving model privacy on fairness \cite{zhang2021balancing}, or the reciprocal effects of enhancing model fairness on privacy \cite{ChangS21}. 
For the method aspect, some methods \cite{zhang2021balancing} have emerged to implement fairness and privacy simultaneously for other AI models.
However, their relationship conclusions or method designs  \cite{XuYW19,HuLWL19,DingZLWYP20,MozannarOS20, 0007FH21,LyuYNLMJYN20,ParkBL22} cannot be directly adapted to GNNs due to the data and model differences.

In the context of GNNs, existing methods \cite{DaiW21, WangGB0023} are not available in this study, as their design goals are different from ours (i.e., individual fairness of nodes and privacy of edges). 
For example, FairGNN \cite{DaiW21} targets the group fairness of nodes and the privacy of sensitive node features; LPF-IFGNN \cite{WangGB0023} expects the individual fairness of nodes and the privacy of node features/labels. 
Table \ref{tab:scope} shows the research differences between our paper and other studies.

\section{Preliminaries }
\label{sec:back}
\noindent
\textbf{Graphs.} 
A graph $G=\{\mathcal{V}, \mathcal{E}\}$ includes an edge set $\mathcal{E}$ and a node set $\mathcal{V}=\left\{v_{1}, \ldots, v_{|\mathcal{V}|}\right\}$. 
$\mathcal{E}$ characterises the relationship information in $G$. 
The set of edges can also be denoted by an adjacency matrix $\mathbf{A}\in\{0,1\}^{|\mathcal{V}| \times |\mathcal{V}|}$, in which $\mathbf{A}_{i,j}=1$ when $e_{ij}=(v_{i},v_{j}) \in \mathcal{E}$, otherwise $\mathbf{A}_{i,j}=0$. 
Matrix $\mathbf{X} \in$ $\mathbb{R}^{|\mathcal{V}| \times k}$ ($k$ indicates the dimensionality of features)
denotes node features, the $i$-th row of $\mathbf{X}$ represents the feature of node $v_{i}$. 
Without loss of generality, another description form of a graph is $G=\{\mathbf{A}, \mathbf{X}\}$. 
In this paper, we focus on the undirected graph.

\noindent
\textbf{Node Classification} and \textbf{GNNs.}
For a graph $G=\{\mathcal{V}, \mathcal{E}\}$, the set of labelled nodes is denoted by $\mathcal{V}_{l} \subset \mathcal{V}$, where $y_{v}$ is the label of $v \in \mathcal{V}_{l}$. 
The set of unlabelled nodes is denoted by $\mathcal{V}_{u} = \mathcal{V} \setminus \mathcal{V}_{l}$. Given $G$ and node labels, node classification aims to train a GNN model, which can predict labels for nodes in $\mathcal{V}_{u}$.
The graph convolutional network (GCN) \cite{KipfW17} is a typical GNN model for node classification. 
In a GCN model, we assume that $\mathbf{E}^{(l)}$ and $\mathbf{W}^{(l)}$ denote the node embeddings and weight matrix of the $l$-th hidden layer, respectively. The graph convolution at the $l$-th layer is formulated as $\mathbf{E}^{(l)}=\sigma(\hat{\mathbf{A}} \mathbf{E}^{(l-1)} \mathbf{W}^{(l)})$, where $\sigma$ is the nonlinear activation, 
$\hat{\mathbf{A}}=\tilde{\mathbf{D}}^{-\frac{1}{2}}(\mathbf{~A}+\mathbf{I}) \tilde{\mathbf{D}}^{-\frac{1}{2}}$, 
and $\tilde{\mathbf{D}}$ being the degree matrix of $(\mathbf{A}+\mathbf{I})$.



\noindent
\textbf{Jaccard Similarity} and \textbf{$k$-hop Node Pairs.}
\textbf{(1)} When using $\mathcal{N}(i)$ to denote the set of connected nodes of $v_i$, $v_j \in \mathcal{N}(i)$ if $\mathbf{A}_{i,j}=1$, otherwise $v_j \notin \mathcal{N}(i)$ (i.e., $\mathbf{A}_{i,j}=0$). The Jaccard similarity matrix $\mathbf{S}$ is derived from $\mathbf{A}$, and each element is defined by $\mathbf{S}_{i,j} = \frac{|\mathcal{N}(i) \cap \mathcal{N}(j)|}{|\mathcal{N}(i) \cup \mathcal{N}(j)|}$.
\textbf{(2)} In this paper, $v_j$ is a $k$-hop neighbour of $v_i$ when the edge number in the shortest path between $v_i$ and $v_j$ is $k$ ($k \in \mathbb{Z}^{+}$), and $(v_i, v_j)$ is called a $k$-hop ($k\geq 1$) node pair. 
Specially, $(v_i, v_j)$ is a $\infty$-hop node pair if $v_i$ and $v_j$ are isolated and there is no path between them.

\noindent
\textbf{Homophily} and \textbf{Sparsity}. \textbf{(1)} For simplicity, we use $n = |\mathcal{V}|$ to represent the node number in $G=\{\mathcal{V}, \mathcal{E}\}$.
To depict the biased node connections, we use $p$ and $q$ to denote the intra-class and inter-class linking probabilities, respectively.
Specifically, given a node $v_i$, for its intra-class neighbour set ${\mathcal{N}(i)}_{\text{intra}}$ and inter-class neighbour set ${\mathcal{N}(i)}_{\text{inter}}$, the set size expectations are $\mathbb{E}[|{\mathcal{N}(i)}_{\text{intra}}|]=(n-1)p$ and $\mathbb{E}[|{\mathcal{N}(i)}_{\text{inter}}|]=(n-1)q$, respectively.
Note that we have $p > q \geq 0$ due to the homophily of graph data.
\textbf{(2)} Comparing with the number of connected node pairs (i.e., $1$-hop node pairs), the number of unconnected node pair is obvious larger as the sparsity of graph data \cite{GNNBook-ch4-tang}, which indicates that $1 > 1-p \gg p$. 

\section{Problem Statement}
\label{sec:problem}
This section introduces the definitions of individual fairness of nodes and privacy risks of edges, followed by presenting the research questions in this paper.

\noindent
\textbf{Individual Fairness of Nodes.} 
In node classification, \textit{individual fairness} requires that \textit{any two similar nodes should receive similar GNN predictions} \cite{KangHMT20}. Concretely, given the (Jaccard) similarity matrix $\mathbf{S}$ of nodes and GNN predictions $\mathbf{Y}$, the bias w.r.t. individual fairness is measured by $Bias(\mathbf{Y},\mathbf{S}) = Tr(\mathbf{Y}^{T}\mathbf{L_{S}Y})$, where $\mathbf{Y}^{T}$ represents the transpose of $\mathbf{Y}$, $\mathbf{L_{S}}$ indicates the Laplacian matrix of $\mathbf{S}$ \cite{KangHMT20}. 
To improve fairness, the InFoRM proposes to involve $Bias(\mathbf{Y},\mathbf{S})$ in the loss function during GNN training \cite{KangHMT20}.
\vspace{-5pt}
\begin{definition}[Individual Fairness of Nodes]
Given the defined
\begin{equation}
    f_{bias}:= Tr(\mathbf{Y}^{T}\mathbf{L_{S}Y}),
\end{equation}
the individual fairness of nodes can be achieved by $\min f_{bias}$ \cite{KangHMT20}.
In particular, observations of decreasing $f_{bias}$, i.e., $ f_{bias} \downarrow$, indicate fairer predictions on nodes.
\end{definition}
\vspace{-5pt}
\noindent
Refer to Section \ref{sec:relatedwork:fairness} for why we focus on Jaccard similarity.

\noindent
\textbf{Privacy Risks of Edges.} 
Existing studies \cite{HeJ0G021,0011L0022} show that attackers are capable of inferring the graph structure in the training graph of a GNN model. 
For example, 
given a well-trained GNN , \textit{link stealing attacks} use the distance $d(v_i,v_j)$ between predictions on $v_i$ and $v_j$ to infer if they are connected, whose intuition is that \textit{if two nodes share more similar predictions from the target GNN model then there is a greater likelihood that the nodes are linked together} \cite{HeJ0G021}. 

Taking the black-box attack setting as an example \cite{HeJ0G021}, attackers only need to query the victim GNN to obtain its predictions (e.g., $GNN(v_i)$ on node $v_i$). Next, the attacker introduces a distance function (e.g., cosine similarity) to obtain the distance between any two nodes (i.e., $d(v_i,v_j)=d(GNN(v_i),GNN(v_j))$) and uses the KNN method to divide all node pairs into two clusters, i.e., the cluster of closer node pairs and the cluster of more distant node pairs. $v_i$ and $v_j$ are inferred to be connected if $d(v_i,v_j)$ is in the cluster with smaller distances; otherwise $v_i$ and $v_j$ are not connected.

Note that, given a $k$-hop node pair $(v_i,v_j)$, $v_i$ and $v_j$ are either unconnected (i.e., $k \geq 2$) or connected (i.e., $k = 1$).
We use $d_{0}=d_{0}(v_i,v_j)$ to denote the prediction distance of $v_i$ and $v_j$ when they are unconnected, otherwise $d_{1}=d_{1}(v_i,v_j)$ when they are connected. 
According to link stealing attacks \cite{HeJ0G021}, the privacy risk of edges is measured by the distinguishability between the distributions of $d_{0}$ and $d_{1}$. 
Thus, following existing studies \cite{GrettonBRSS06, BalunovicRV22}, we propose the definition of privacy risks of edges as following:
\vspace{-5pt}
\begin{definition}[Privacy Risks of Edges]
As the statistical distance between distributions, the privacy risks of edges under link stealing attacks is defined as
\begin{equation}
    f_{risk}:=\overline{\Delta d} =  \|\overline{d_{0}} - \overline{d_{1}}\|,
\end{equation}
where $\overline{d_{0}} = \mathbb{E}[d_{0}] $ and $\overline{d_{1}} = \mathbb{E}[d_{1}]$ denote the expectation of $d_{0}$ and $d_{1}$, respectively.
Particularly, observations of increasing $\overline{\Delta d}$, i.e., $ f_{risk} \uparrow$, indicate the more easily detectable distinguishability between connected and unconnected node pairs.
\end{definition}
\vspace{-5pt}

\noindent
\textbf{Research Questions.} Our study is motivated by the potential connection between the goal of individual fairness (i.e., \textit{decreasing the prediction difference between similar nodes}) and the intuition of link stealing attacks (i.e., \textit{nodes with more similar predictions are prone to be connected}), which indicates that improving individual fairness may lead to more serious privacy risks on edges. 
Thus, our research questions (RQ) are:
\\ \noindent
\textbf{RQ1}: (${f_{bias} \downarrow} \Rightarrow f_{risk} ?$) ``\textit{Is individual fairness of nodes potentially at the cost of privacy risks of edges}?'' If so, the following question is 
\\ \noindent
\textbf{RQ2}: ($f_{bias}$ \& $f_{risk} \downarrow $ simultaneously?) ``\textit{Could we devise a method that can alleviate the adverse effect on privacy risks of edges when improving individual fairness of nodes}?''

\section{Trade-off between Individual Fairness and Privacy Risks}
\label{sec:tradeoff}
To answer \textbf{RQ1}, after discussing the value in the Jaccard similarity matrix, we perform a theoretical analysis to connect the individual fairness and privacy risks. 
\begin{lemma}
    Given a graph $G=\{\mathbf{A}, \mathbf{X}\}$, for the Jaccard similarity $\mathbf{S}$ derived from $\mathbf{A}$, we have 
    \begin{equation}
    \label{eq:S_value}
        \mathbf{S}_{i,j} 
        \begin{cases}
        > 0 & \text{if\ \ } k \leq 2, \\
        =0 & \text{if\ \ } k > 2.        
        \end{cases}
    \end{equation}
    where $k$ indicates that $(v_i, v_j)$ is a $k$-hop ($k\geq 1$) node pair.
\end{lemma}
\begin{proof}
According to $\mathbf{S}_{i,j} = \frac{|\mathcal{N}(i) \cap \mathcal{N}(j)|}{|\mathcal{N}(i) \cup \mathcal{N}(j)|}$, we only discuss the value of $|\mathcal{N}(i) \cap \mathcal{N}(j)|$ as ${|\mathcal{N}(i) \cup \mathcal{N}(j)| > 0}$ always stands.
\\ \noindent
(1) $k = 1$: The normalisation operation $\hat{\mathbf{A}}=\tilde{\mathbf{D}}^{-\frac{1}{2}}(\mathbf{~A}+\mathbf{I}) \tilde{\mathbf{D}}^{-\frac{1}{2}}$ causes $\{v_i,v_j\} \in \mathcal{N}(i)$ and $\{v_j,v_i\} \in \mathcal{N}(j)$, which indicates that $|\mathcal{N}(i) \cap \mathcal{N}(j)|\geq 2>0$.
\\ \noindent
(2) $k = 2$: At least one node $v_k$ exists on the shortest path between $v_i$ and $v_j$, i.e., $v_k \in \mathcal{N}(i)$ and $v_k \in \mathcal{N}(j)$, which indicates that $|\mathcal{N}(i) \cap \mathcal{N}(j)|\geq 1 >0$.
\\ \noindent
(3) $k > 2$: $|\mathcal{N}(i) \cap \mathcal{N}(j)|=0$ according to the definition of $k$-hop neighbours.
\end{proof}

\noindent
The above lemma leads to the following proposition.
\begin{framed}
\vspace{-10pt}
\begin{proposition}
\label{eq:proposition}
    For a GNN model trained on a homophily and sparse graph, improving individual fairness of nodes is at the cost of privacy risks of edges, i.e., 
    \begin{equation*}
        f_{bias}\downarrow \quad \Rightarrow \quad f_{risk} \uparrow
    \end{equation*}
\end{proposition}
\vspace{-10pt}
\end{framed}
\vspace{-8pt}
\begin{proof} 
To derive it, we present related notations (e.g., $k$-hop node pairs, homophily), followed by connecting individual fairness of nodes and privacy risks of edges.

\noindent
(1) $f_{bias} \Rightarrow d $. When improving the fairness by $\min f_{bias}$, for a node pair $(v_i,v_j)$, $d(v_i,v_j)$ is decreased only if $\mathbf{S}_{i,j}>0$ according to the definition of the Laplacian matrix \cite{PangC17, KangHMT20}. Thus, improving GNN fairness has the following different impacts on node pairs with different hop values. Specifically, we have
\begin{equation}
\label{eq:fbias2d}
    \min f_{bias} \Rightarrow f_{bias} \downarrow \Rightarrow
    \begin{cases}
    d_{1}(v_i,v_j) \downarrow  & \text{if } k=1, \\
    d_{0}(v_i,v_j) \downarrow  & \text{if } k=2, \\
    d_{0}(v_i,v_j) \rightarrow & \text{if } k > 2,        
    \end{cases}
\end{equation}
where the $\downarrow$ and $\rightarrow$ represent obviously decreasing and near-invariant distance values, respectively.

\noindent
(2) $ d \Rightarrow  f_{risk}$. The above value changes impacts the $f_{risk}$ in the following two aspects:

\textit{(a)} Decreased $\overline{d_{1}}$. We obtain it (i.e., $\overline{d_{1}} \downarrow$) since $\forall$ connected (i.e., $1$-hop) node pair, we have $d_{1}(v_i,v_j) \downarrow$, as shown in (\ref{eq:fbias2d}).

\textit{(b)} Near-invariant $\overline{d_{0}}$. According to (\ref{eq:fbias2d}), $2$-hop node pairs are the only affected part in calculating $\overline{d_{0}}$ when improving GNN fairness. 
However, $2$-hop node pairs are only an extremely minor part due to the sparsity of graph data. 
Considering a two-class classification task for simplicity, for a specific node $v_i$, the number of all connected node pairs is 
$\#_{\text{con}}^{\text{all}} = \#_{\text{con}}^{\text{1-hop}} = (n-1)p + (n-1)q = (n-1)(p+q)$,
and the ratio of $2$-hop node pairs in calculating $\overline{d_{0}}$ is
\begin{equation}
\begin{aligned}
    ratio & = \frac{\#_{\text{uncon}}^{\text{2-hop}}}{\#_{\text{uncon}}^{\text{all}}} 
    = \frac{\#_{\text{uncon}}^{\text{2-hop}}}{(n-1) - \#_{\text{con}}^{\text{all}}} \\
    & = \frac{(n-1)pp+(n-1)pq+(n-1)qp+(n-1)qq}{(n-1)(1-p-q)} \\
    & = \frac{(p+q)^2}{1-(p+q)}.
\end{aligned}
\end{equation}
We obtain that the above ratio is almost $0$ due to the sparsity (i.e., $1 > 1-p \gg p$) and homophily (i.e., $p>q$) of graph data. 
Although it shows that $ d_{0}(v_i,v_j) \downarrow$ in (\ref{eq:fbias2d}) when $k=2$, the extremely minor role of $2$-hop node pairs results in a near-invariant $\overline{d_{0}}$ (i.e., $\overline{d_{0}} \rightarrow$).
 
Given the near-invariant $\overline{d_{0}}$ and decreased $\overline{d_{1}}$, we obtain that $\overline{\Delta d} =  \|\overline{d_{0}} - \overline{d_{1}}\|$) is increased, which indicates $f_{risk} \uparrow$.  
\end{proof}

Moreover, ${f_{bias}\downarrow} \Rightarrow {f_{risk}\uparrow}$ is also confirmed by our empirical evaluations in Section \ref{sec:exp:pre}, which indicates the trade-off between fairness and privacy of GNNs.


\section{Balancing Fairness and Privacy with limited performance cost}
\label{sec:method}
\begin{figure}
  \centering
  \includegraphics[width=0.8\linewidth]{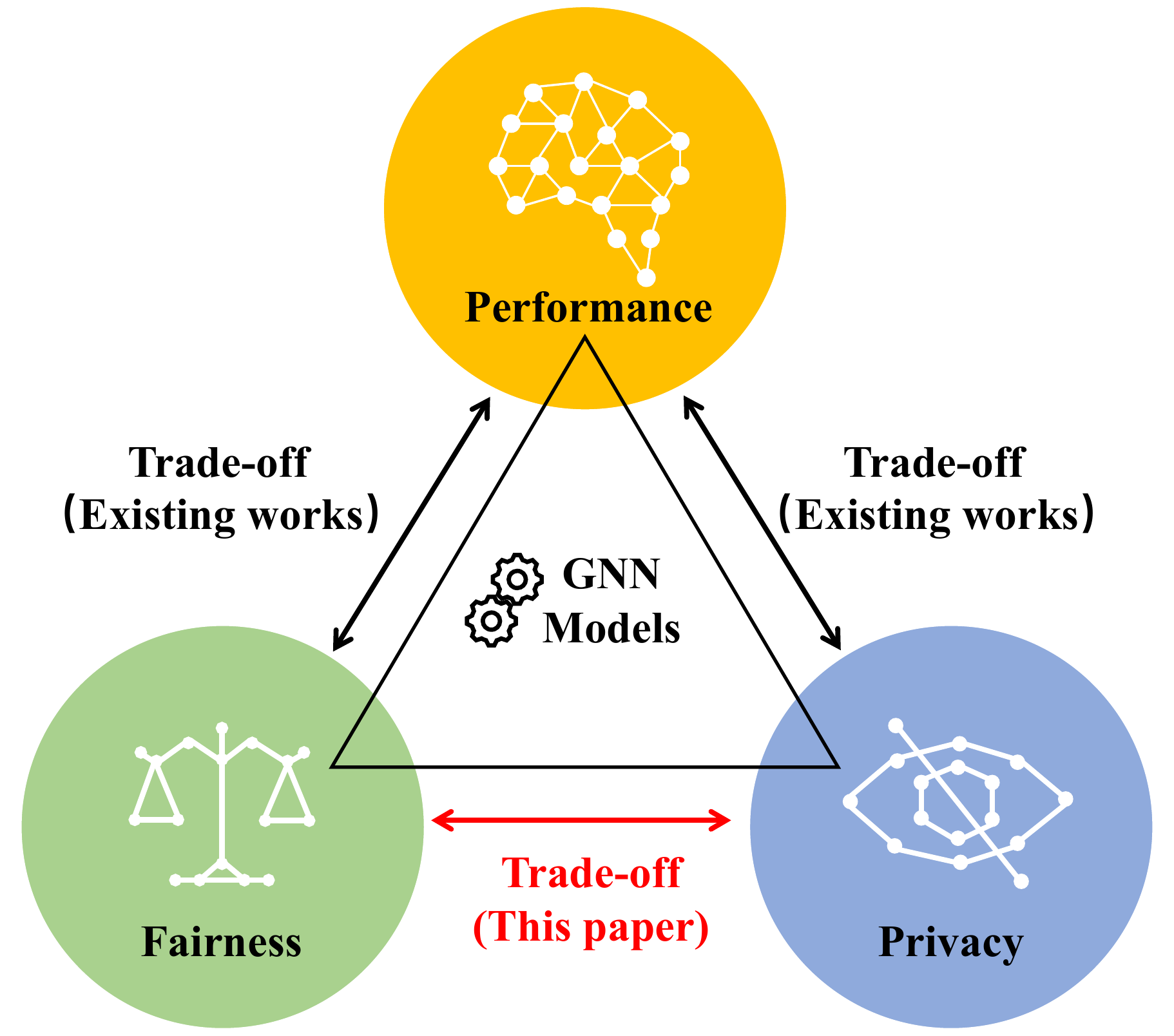}
  \vspace{-5pt}
  \caption{In addition to the existing trade-off between performance and fairness/privacy, the trade-off between fairness and privacy is demonstrated in the context of GNNs.}
  \label{fig:interactions}
\end{figure}

Building trustworthy GNNs requires considering performance and several aspects of trustworthiness simultaneously \cite{abs-2205-07424}, however, taking performance, fairness, and privacy into consideration simultaneously is not trivial \cite{GuZLZRC22}.
This is because, in addition to the finding that both the fairness \cite{DongKTL21} and the privacy \cite{0011L0022} of GNNs are at the cost of performance, our study demonstrates that there is a trade-off between individual fairness of nodes and privacy risks of edges, as shown in Fig. \ref{fig:interactions}.
In this paper, we argue that the performance of GNNs lies in their central position to serve users, even if when improving fairness or decreasing privacy risks.
That is, the weight ($\omega$) relations of different goals in comprehensively building trustworthy GNNs are $\omega (Accuracy) > \omega (Fairness) = \omega (Privacy)$.
Thus, to answer \textbf{RQ2}, \textbf{our goal} is to devise a method that can boost \textit{\ul{fairness}} with limited \textit{\ul{performance}} costs and restricted \textit{\ul{privacy}} risks. 

\begin{figure*}[t!]
  \centering
  \includegraphics[width=0.95\linewidth]{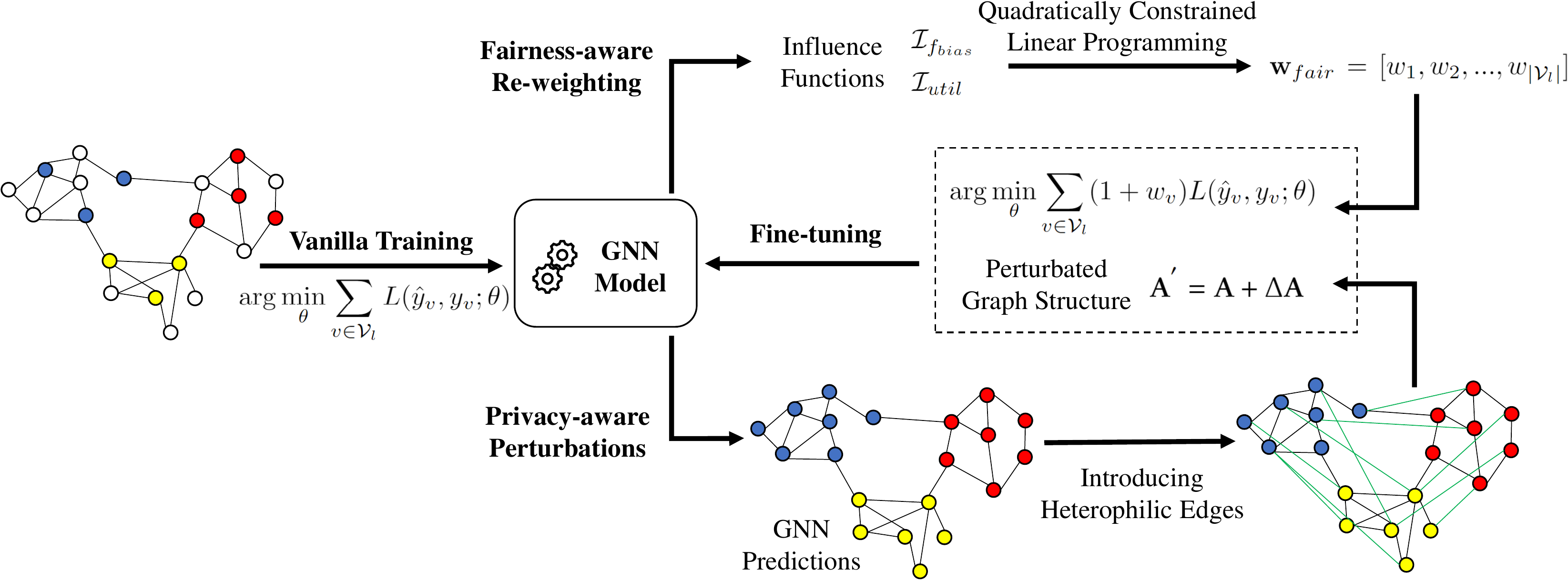}
  \caption{The framework of privacy-aware perturbations and fairness-aware re-weighting (PPFR) method. After the \textit{vanilla training} of a GNN model, the privacy-aware module introduces heterophily edges to generate the perturbed graph structure, and the fairness-aware reweighting module employs the influence function and quadratically constrained linear programming (QCLP) to obtain fairness-aware weighted loss. Then, the perturbed graph structure and fairness-aware weighted loss are involved in \textit{ fine-tuning} to promote fairness with limited performance cost and restricted privacy risks.}
  \label{fig:framework}
  \vspace{-10pt}
\end{figure*}

To this end, we propose the Privacy-aware Perturbations and Fairness-aware Reweighting (PPFR) method.
As shown in Fig. \ref{fig:framework}, the entire training of the target GNN includes two phases: \textit{vanilla training} and \textit{PPFR fine-tuning}. 
Specifically, we first conduct \textit{vanilla training} to obtain a competent GNN model (\textit{\ul{performance}}).
After that, we continue to train (i.e., \textit{PPFR fine-tuning}) the target GNN with the perturbed graph structure (\textit{\ul{privacy}}) and the weighted loss function derived from fairness-aware weights (\textit{\ul{fairness}}). 
Our PPFR method is \textit{model-agnostic} and can be used after the vanilla training of any GNN models, allowing it to contribute to multi-aspects of trustworthy GNNs in a \textit{plug-and-play} manner. 

The intuition of PPFR is that, given the fixed model architecture, the behaviour of a GNN model is influenced by both the training data and the loss function, where specially designed training strategy can improve its trustworthiness w.r.t. our design goal. 
Considering that incorporating operations to improve fairness or privacy during the early stages of training potentially harms GNN performance, we propose the two-phase framework (i.e., vanilla training and fine-tuning) to make the vanilla training concentrate on improving GNN performance while the fine-tuning focus on enhancing fairness and privacy.
In this section, we first introduce how to evaluate the influence of training samples. Next, we present the fairness-aware reweighting (FR) and privacy-aware perturbations (PP), followed by discussions on our design.

\subsection{The Influence of Training Samples}
\label{sec:method:influence}
Current works \cite{KangHMT20, 0011L0022} show that GNN behaviours (e.g., bias or privacy leakage) are related to the model parameters, which are affected by GNN training samples. In this section, we first present how the weight of training samples influences the parameters of models, and then discuss its influence on interested functions (e.g., GNN loss function, metric function of bias, or privacy risks).
\subsubsection{Influence of Training Samples on GNN Parameters}
Generally, given a set of labelled nodes in $G$, we can train a GNN model by minimising the following loss function \cite{GNNBook2022}:
\begin{equation}
\label{eq:aveloss}
    \theta^{*}(\mathbf{1}) = \arg \min_{\theta} \sum_{v \in \mathcal{V}_{l} } L (\hat{y}_{v}, y_v;\theta),
\end{equation}
where $y_v$ represents the ground truth label of node $v \in \mathcal{V}_{l}$, and $\hat{y}_v$ indicates the predicted label from the GNN model with parameter $\theta$. 
The $\mathbf{1}$ (i.e., all-one vector) indicates that all nodes in $\mathcal{V}_{l}$ are treated equally when training GNNs.
When changing the weight of training samples, the obtained parameter can be expressed as
\begin{equation}
\label{eq:weightedloss}
    \theta^{*}(\mathbf{1}+\mathbf{w}) = \arg \min_{\theta} \sum_{v \in \mathcal{V}_{l} } (1+w_v)L (\hat{y}_{v}, y_v;\theta),
\end{equation}
where $w_v\in [-1,1]$ and $(1+w_v)$ represents the weight of node $v$ in the loss function when training a GNN model (e.g., $w_v=-1$ indicates leaving node $v$ out of training phase). To estimate $\theta^{*}(\mathbf{1}+\mathbf{w})$ without retraining the GNN model, here we employ the influence function \cite{KohL17,KangZT22} and Taylor expansion to conduct the following first-order approximation:
\begin{equation}
\label{eq:theta1plusw}
    \theta^{*}(\mathbf{1}+\mathbf{w}) \approx \theta^{*}(\mathbf{1})+ \sum_{v \in \mathcal{V}_{l}} w_v \mathcal{I}_{\theta^{*}(\mathbf{1})}(v),
\end{equation}
where $\mathcal{I}_{\theta^{*}(\mathbf{1})}(v) = \frac{d \theta^{*}(\mathbf{1})}{d w_v}|_{w_v=0}$ is the influence function w.r.t. node $v$, and can be calculated as 
\begin{equation}
    \mathcal{I}_{\theta^{*}(\mathbf{1})}(v) = \mathbf{H}_{\theta^{*}(\mathbf{1})}^{-1}\nabla_{\theta}L (\hat{y}_{v}, y_v;\theta^{*}(\mathbf{1})).
\end{equation}
$\mathbf{H}_{\theta^{*}(\mathbf{1})} = \frac{1}{|\mathcal{V}_{l}|} \sum_{v \in \mathcal{V}_{l} } \nabla_{\theta}^{2} L (\hat{y}_{v}, y_v;\theta^{*}(\mathbf{1}))$ is the Hessian matrix of loss function w.r.t $\theta^{*}(\mathbf{1})$.

\subsubsection{Influence on Interested Functions}
To study the behaviour of GNN models, different $f$ functions (e.g., $f_{bias}$ and $f_{risk}$) have been proposed to evaluate the GNN outputs. 
In this paper, we use the Taylor expansion of $f$ with respect to the parameters $\theta$ and the equation (\ref{eq:theta1plusw}) to calculate the influence of training samples on $f$. Specifically, we assume that $f_{bias}$ is induced by the parameter $\theta$ of the target GNN since the input of $f_{bias}$ is the final prediction of the target GNN. Therefore, the influence of training samples on a interested function $f$ can be expressed as
\begin{equation}
\label{eq:Ifw}
\begin{aligned}
    \mathcal{I}_{f}(\mathbf{w}) & = \frac{df}{d\theta}\frac{d\theta}{d \mathbf{w}} \approx \nabla_{\theta}f(\theta^{*}(\mathbf{1}))^{T} \left[\theta^{*}(\mathbf{1}+\mathbf{w}) - \theta^{*}(\mathbf{1})\right]\\
    & \approx \nabla_{\theta}f(\theta^{*}(\mathbf{1}))^{T} \left[\sum_{v \in \mathcal{V}_{l}} w_v\mathcal{I}_{\theta^{*}(\mathbf{1})}(v)\right].
\end{aligned}
\end{equation}
Note that any derivable function that takes the GNN prediction $\mathbf{Y}$ as input can be considered as $f$. For example, $f$ can be instantiated as the loss function that concerns the utility (i.e., performance) of the target model, i.e., 
\begin{equation}
    \mathcal{I}_{util}(\mathbf{w}) =  \sum_{v \in \mathcal{V}_{l}} \nabla_{\theta}L (\hat{y}_{v}, y_v;\theta^{*}(\mathbf{1}))^{T} \left[\sum_{v \in \mathcal{V}_{l}} w_v\mathcal{I}_{\theta^{*}(\mathbf{1})}(v)\right].
\end{equation}
Moreover, when using the bias-related function $f_{bias}(\theta)$ and the privacy risk-related function $f_{risk}(\theta)$, we can obtain the following influence of training samples on GNN fairness and privacy risks:
\begin{equation}
\begin{aligned}
    \mathcal{I}_{f_{bias}}(\mathbf{w}) &=
    \nabla_{\theta}f_{bias}(\theta^{*}(\mathbf{1}))^{T} \left[\sum_{v \in \mathcal{V}_{l}} w_v\mathcal{I}_{\theta^{*}(\mathbf{1})}(v)\right], \\
    \mathcal{I}_{f_{risk}}(\mathbf{w}) &=
    \nabla_{\theta}f_{risk}(\theta^{*}(\mathbf{1}))^{T} \left[\sum_{v \in \mathcal{V}_{l}} w_v\mathcal{I}_{\theta^{*}(\mathbf{1})}(v)\right].
\end{aligned}
\end{equation}
In this paper, we use $\mathbf{w}_v$ to denote an all-zero vector, except $w_v$ is -1. Thus, $\mathcal{I}_{f}(\mathbf{w}_v)$ represents leaving node $v$ out of the training of a GNN model, that is, $\mathcal{I}_{f}(\mathbf{w}_v) =-\nabla_{\theta}f(\theta^{*}(\mathbf{1}))^{T} \mathbf{H}_{\theta^{*}(\mathbf{1})}^{-1} \nabla_{\theta}L (\hat{y}_{v}, y_v;\theta^{*}(\mathbf{1}))$.

\textit{Remarks.} 
Edges between nodes break the IID assumptions among nodes when calculating influence functions on graph data, and a recent study on GNNs \cite{DongWMLL2022} proposes to calculate node influence by considering the existence of edges in graph data. 
However, comparing with retaining GNNs by equation (\ref{eq:weightedloss}), its experimental evaluations confirm that edges can only cause extremely limited estimation degradation on $\theta^{*}(\mathbf{1}+\mathbf{w})$.
As a result, we use the calculation in this section for efficiency without sacrificing the accuracy when estimating $\theta^{*}(\mathbf{1}+\mathbf{w})$.

\subsection{Boosting Fairness with Restricted Privacy Risks}
\subsubsection{Fairness-aware Re-weighting (FR)}
\label{sec:method:ppfr:fr}
According to a recent GNN fairness study \cite{DongWMLL2022}, bias in a GNN model can be interpreted by node attribution, and fairness can be improved by eliminating harmful training nodes. 
Inspired by this work, unlike directly deleting training nodes and then training target GNNs from scratch, we propose to reweight the training nodes in the loss function of a GNN after it vanilla training.
Specifically, the weight can be obtained by the following Quadratically Constrained Linear Programming (QCLP),
\begin{equation}
\label{equ:wfair}
\begin{array}{ll}
\min_{w} &  \sum\limits_{v \in \mathcal{V}_{l}} w_v \mathcal{I}_{f_{bias}}(\mathbf{w}_v) \\
\mathnormal{ s.t. } 
& \sum\limits_{v \in \mathcal{V}_{l}} w_v^{2} \leq \alpha |\mathcal{V}_{l}|, \\
& \sum\limits_{v \in \mathcal{V}_{l}} w_v  \mathcal{I}_{f_{util}}(\mathbf{w}_v) \leq \beta \sum \mathcal{I}_{f_{util}}^{+}(\mathbf{w}_v),\\
& -1 \leq w_v \leq 1.
\end{array}
\end{equation}
In (\ref{equ:wfair}), $w_v$ is the variable and the objective tends to minimise the total bias existing in GNN predictions. The first constraint is designed to control the degree of reweighting. The second constraint represents that the obtained weights only cost limited the utility of the model, where $\mathcal{I}_{f_{util}}^{+}(\mathbf{w}_v)$ indicates $\mathcal{I}_{f_{util}}(\mathbf{w}_v)$ with positive value. $\alpha$ and $\beta$ are hyperparameters. 

Note that our QCLP optimisation (\ref{equ:wfair}) is different from the training scheme Linear Programming (LP) in \cite{LiL22a} since our QCLP has different constraints and a wider search space, which leads to better debiasing effects and lower performance cost. 
After solving this QCLP with the Gurobi optimiser \cite{Gurobi}, the obtained weight $\mathbf{w}_{fair}=[w_1, w_2, ..., w_{|\mathcal{V}_{l}|}]$ can be used to fine-tune the target GNN model with weighted loss in the equation (\ref{eq:weightedloss}).

\noindent
\textbf{Why not considering ${\mathcal{I}_{f_{risk}}(\mathbf{w}_v)}$ in QCLP?} 
Given the proposed re-weighting method, a straightforward approach of taking privacy and fairness into account simultaneously is that involving an item that concerns $\mathcal{I}_{f_{risk}}$ into the QCLP (e.g., playing a role as a constraint or part of the objective). 
However, the negative correlation between $\mathcal{I}_{f_{bias}}$ and $\mathcal{I}_{f_{risk}}$ leads to the weak effectiveness of this straightforward method.

To verify this method, we first calculate the influence on $f_{bias}$ and $f_{risk}$, and concatenate all $\mathcal{I}_{f}(\mathbf{w}_v)$ $(v \in \mathcal{V}_{l})$ together to obtain the influence vector ${\mathcal{I}_{f_{bias/risk}}\in \mathbb{R}^{1 \times |\mathcal{V}_{l}|}}$.
Next, we employ the Pearson coefficient $r$ between $\mathcal{I}_{f_{bias}}$ and $\mathcal{I}_{f_{risk}}$ to measure the interactions between fairness and privacy. Specifically, 
$r=\mathtt{Pearson}(\mathcal{I}_{f_{bias}}, \mathcal{I}_{f_{risk}})$, 
where $-1 \leq r \leq 1$. In the context of measuring interactions between fairness and privacy, $r = 1$ represents a mutual promotion, while $r = -1$ indicates that there is a conflict between them. 

As shown in Table \ref{tab:correlation}, $r < 0.3$ indicates the inconformity \cite{akoglu2018user} between $\mathcal{I}_{f_{bias}}$ and $\mathcal{I}_{f_{risk}}$, which indicates that it is difficult to obtain effective node weights in equation (\ref{equ:wfair}). 
Moreover, our empirical studies also confirm the weak effectiveness of this straightforward method.
Note that $f_{risk}$ is instantiated as $f_{risk}(\theta) = \frac{2\|\overline{d_{0}} - \overline{d_{1}}\|}{var(d_{0})+var(d_{1})}$ for better estimation accuracy.

\begin{table}
\centering
\caption{Correlation $r$ between $\mathcal{I}_{f_{bias}}$ and $\mathcal{I}_{f_{risk}}$.}
\label{tab:correlation}
\scalebox{1}{
\begin{tabular}{|c|c|c|c|}
\hline
\multicolumn{1}{|l|}{} & GCN   & GAT   & GraphSage \\ \hline
Cora                   & -0.66 & 0.26  & 0.87      \\ \hline
Citeseer               & -0.51 & -0.96 & -0.24     \\ \hline
Pubmed                 & -0.41 & 0.52  & 0.25      \\ \hline
\end{tabular}
}
\vspace{-10pt}
\end{table}

\subsubsection{Privacy-aware Perturbations (PP)}
\label{sec:method:ppfr:pp}
Although our analysis in Section \ref{sec:tradeoff} confirms the adverse effect (increased privacy risks of edges) of improving individual fairness of nodes, it also indicates the opportunity to navigate the fairness and privacy, as they focus on different spaces of node pairs (e.g., improving individual fairness leads to reduced $\overline{d_{1}}$, while privacy risks can be restricted by reducing $\overline{d_{0}}$). 
Considering the inconformity (i.e., $r < 0.3$ in Table \ref{tab:correlation}) between fairness and privacy in the weight space of the loss function, we propose to restrict the privacy risks of edges by modifying the graph data when improving fairness with the reweighted loss function. 
In this section, we first present a risk modeling to reveal potential factors that contribute to the privacy risk of edges, followed by showing our privacy-aware perturbation method.

\noindent
\textbf{Modeling privacy risks of edges.} 
To dig out factors that contributes to $f_{risk}$, we implement a preliminary analysis on synthetic $G$ whose nodes can be categorised into two classes (i.e., class-0 and class-1). 
Instead of directly analysing privacy risks in the link score space, we will study it in the embedding space and show how the existence of an edge impacts the learnt embedding. 
In the embedding space, a well-trained GNN model clusters similar nodes together so that they have similar predictions. 
According to previous studies \cite{PanHLJYZ18, DongLJL22}, we assume that the learnt node embedding follows the normal distribution. 
For simplicity, we employ the left normalisation $\hat{\mathbf{A}}=\tilde{\mathbf{D}}^{-1}(\mathbf{~A}+\mathbf{I})$ in GCN models and consider the binary node classification task. 
Assuming that $\mu_i$ and $\sigma$ indicate the mean and standard deviation of the node embedding from class $y_i$ ($i=0,1$) at the $t$-th layer, we obtain the learnt embedding $\mathbf{E}^{(t, y_i)} \sim N(\mu_i, {\sigma}^2)$.

In this paper, we focus on analysing the intra-class node pairs (i.e., nodes with the same label), since they are the majority of all node pairs in binary classifications \cite{ZhangWWYXPY23}.
In a GCN model, the edges are involved in the one-hop mean-aggregation operation, i.e., $\hat{\mathbf{A}} \mathbf{E}$.
From the view of an individual node $v_i$, the one-hop mean aggregation operation is
\begin{equation}
    {[\hat{\mathbf{A}} \mathbf{E}^{(t)}]_{i} = \frac{1}{d_i+1} (\mathbf{E}^{(t)}_i +  \sum_{ \substack{\scriptscriptstyle v_j \in \mathcal{N}(v_i)\\ m_j=0 }}  \mathbf{E}^{(t,y_0)}_j +  \sum_{ \substack{\scriptscriptstyle v_j \in \mathcal{N}(v_i)\\ m_j=1 }}  \mathbf{E}^{(t,y_1)}_j)},
\end{equation}
where $\mathcal{N}(v_i)$ represents the set of neighbour nodes of $v_i$, and ${{m}_{i}=1}$ indicates $v_{i}$ is from class $y_1$, otherwise ${{m}_{n}=0}$. Given ${\mathbf{m}=[m_1,...,m_{|\mathcal{V}|}]}$, ${\mathbf{E}^{(t,y_1)} = diag(\mathbf{m})\mathbf{E}^{(t)}}$, ${\mathbf{E}^{(t,y_0)} = (\mathbf{I}-diag(\mathbf{m}))\mathbf{E}^{(t)}}$, and ${\mathbf{E}^{(t)}=\mathbf{E}^{(t,y_0)}+\mathbf{E}^{(t,y_1)}}$.
Thus, we have ${\sum_{ \substack{v_j \in \mathcal{N}(v_i)\\ m_j=0 }}  \mathbf{E}_j^{(t,y_0)} \sim N(d_i^{y_0} \mu_0, d_i^{y_0} {\sigma}^2)}$ and ${\sum_{ \substack{v_j \in \mathcal{N}(v_i)\\ m_j=1 }} \mathbf{E}_j^{(t,y_1)} \sim N(d_i^{y_1} \mu_1, d_i^{y_1} {\sigma}^2)}$, where ${d_i^{y_0/y_1}}$ denotes the number of nodes in ${\mathbf{E}^{y_{0}/y_{1}}}$ and ${d_i = d_i^{y_0}+d_i^{y_1}}$ indicates the degree of $v_i$.
Without loss of generality, we assume that $v_i$ and $v_j$ in the node pair $(v_i,v_j)$ come from class 0.

As shown in the following, we analyse how the existence of an edge influences the privacy risk $f_{risk}$. 
Specifically, the distance sensitivity in the embedding space can be calculated as the difference between the cases where $v_i$ and $v_j$ are connected or unconnected.

\noindent
\textit{Case 0}: When $v_i$ and $v_j$ are unconnected, ${\left(\hat{\mathbf{A}} \mathbf{E}^{(t)}\right)_{i}^{0}}$ and ${\left(\hat{\mathbf{A}} \mathbf{E} ^{(t)}\right)_{j}^{0}}$ can be approximately expressed as
\begin{equation}
\label{eq:case0}
\begin{aligned}
    \left(\hat{\mathbf{A}} \mathbf{E} ^{(t)}\right)_{i}^{0} \approx & \frac{1}{d_i+1} \left(\mathbf{E}^{(t)}_i + d_i^{y_0} \mu_0 + d_i^{y_1} \mu_1 \right),\\
    \left(\hat{\mathbf{A}} \mathbf{E} ^{(t)}\right)_{j}^{0} \approx & \frac{1}{d_j+1} \left(\mathbf{E}^{(t)}_j + d_j^{y_0} \mu_0 + d_j^{y_1} \mu_1 \right).
\end{aligned}
\end{equation}
\noindent
\textit{Case 1}:
When $v_i$ and $v_j$ are connected, $\left(\hat{\mathbf{A}} \mathbf{E}^{(t)}\right)_{i}^{1}$ and $\left(\hat{\mathbf{A}} \mathbf{E} ^{(t)}\right)_{j}^{1}$ can be approximately expressed as
\begin{equation}
\label{eq:case1}
\begin{aligned}
    \left(\hat{\mathbf{A}} \mathbf{E} ^{(t)}\right)_{i}^{1} \approx & \frac{1}{d_i+2} \left(\mathbf{E}^{(t)}_i + \mathbf{E}^{(t)}_j + d_i^{y_0} \mu_0 + d_i^{y_1} \mu_1 \right),\\
    \left(\hat{\mathbf{A}} \mathbf{E} ^{(t)}\right)_{j}^{1} \approx & \frac{1}{d_j+2} \left(\mathbf{E}^{(t)}_j + \mathbf{E}^{(t)}_i + d_j^{y_0} \mu_0 + d_j^{y_1} \mu_1 \right).
\end{aligned}
\end{equation}
Given (\ref{eq:case0}) and (\ref{eq:case1}), the distance of $v_i$ and $v_j$ in the embedding space is calculated as 
\begin{equation}
\begin{aligned}
    d_{0}(v_i,v_j) &= \left(\hat{\mathbf{A}} \mathbf{E}^{(t)}\right)_{i}^{0} - \left(\hat{\mathbf{A}} \mathbf{E} ^{(t)}\right)_{j}^{0},\\
    d_{1}(v_i,v_j) &= \left(\hat{\mathbf{A}} \mathbf{E}^{(t)}\right)_{i}^{1} - \left(\hat{\mathbf{A}} \mathbf{E} ^{(t)}\right)_{j}^{1}.
\end{aligned}
\end{equation}
Thus, the sensitivity of $d(v_i,v_j)$ with respect to the existence of edge $e_{ij}$ is 
\begin{equation}
\begin{aligned}
    \Delta d(v_i,v_j) & =  \|d_{0}(v_i,v_j) - d_{1}(v_i,v_j)\| 
    \\ & =   \left\| \frac{\left(\hat{\mathbf{A}} \mathbf{E} ^{(t)}\right)_{i}^{0} - \mathbf{E}^{(t)}_j }{d_i+2}  - \frac{\left(\hat{\mathbf{A}} \mathbf{E} ^{(t)}\right)_{j}^{0} - \mathbf{E}^{(t)}_i}{d_j+2}  \right\|.
\end{aligned}
\end{equation}
Considering that we have
\begin{equation}
\begin{aligned}
    \mathbb{E}\left[\left(\hat{\mathbf{A}} \mathbf{E}^{(t)}\right)_{i}^{0} -\mathbf{E}^{(t)}_j \right] &= \frac{(1+d_i^{y_0})\mu_0+d_i^{y_1}\mu_1}{d_i+1} - \mu_0 \\ &=\frac{d_i^{y_1}(\mu_1-\mu_0)}{d_i+1}, \\
    \mathbb{E}\left[\left(\hat{\mathbf{A}} \mathbf{E} ^{(t)}\right)_{j}^{0} -\mathbf{E}^{(t)}_i \right] &= \frac{(1+d_j^{y_0})\mu_0+d_j^{y_1}\mu_1}{d_j+1} - \mu_0\\
    &=\frac{d_j^{y_1}(\mu_1-\mu_0)}{d_j+1}
\end{aligned}
\end{equation}
the expectation of $\Delta d(v_i,v_j)$  is
\begin{equation}
\label{eq:frisk:factors}
    \mathbb{E}\left[ \Delta d(v_i,v_j) \right] = \left\|(\mu_1-\mu_0)\delta\right\|,
\end{equation}
where $\delta =\frac{d_i^{y_1}}{(d_i+1)(d_i+2)} -  \frac{d_j^{y_1}}{(d_j+1)(d_j+2)}$.  

\vspace{3pt}
The existing study on GNNs \cite{DongLJL22} has demonstrated that preliminary analysis on synthetic datasets helps to devise effective methods to build trustworthy GNNs.
The derived result reveals potential factors that contribute to $f_{risk}$. 
For example, the value of $\|\mu_1-\mu_0\|$ reflects the ability of GNNs to distinguish two classes of nodes, (\ref{eq:frisk:factors}) indicates that a GNN model with higher performance (i.e., owning a larger $\|\mu_1-\mu_0\|$) has higher edge leakage risks due to the homophily of graph data (i.e., connected nodes are more likely to have similar attributes and the same label). 


\noindent
\textbf{Modifying graph structure}. 
The modification strategy is led by the following two insights derived from our theoretical analysis.
(1) According to Section \ref{sec:tradeoff}, reducing $\overline{d_{0}}$ helps to restrict the increased privacy risks of edges (i.e., $ f_{risk}=\|\overline{d_{0}} - \overline{d_{1}}\|$) where $d_{0}(v_i,v_j)$ is reduced. 
(2) Our analysis of privacy risk in (\ref{eq:frisk:factors}) indicates that reducing the node distance between different classes (i.e., $\|\mu_1-\mu_0\|$) benefits privacy of edges in GNNs.
Therefore, we proposed to introduce heterophilic noisy edges into the original graph structure $\mathbf{A}$, which is consistent with the above two insights. 


Concretely, the modified graph structure $\mathbf{A}^{'}$ is obtained by adding privacy-aware perturbations $\Delta \mathbf{A}$ to $\mathbf{A}$, i.e.,
\begin{equation}
\label{eq:riskpeturb}
    \mathbf{A}^{'} = \mathbf{A} + \Delta \mathbf{A}.
\end{equation}  
$\Delta \mathbf{A}$ is first initiated as an all-zero matrix, the generation of nonzero elements is guided by the well-train GNN obtained from the vanilla training phase, which is expected to have excellent performance (e.g., high accuracy) in node prediction.
Taking node $v_i \in \mathcal{V}$ as an example, we use $\mathbf{Y}_{i}=GNN(v_i)$ to indicate the predicted label of $v_i$, and its unconnected node set is denoted by $\mathcal{V}\backslash\mathcal{N}(i)$.  
We randomly select nodes (e.g., $v_j$) with different labels (i.e., $\mathbf{Y}_{i} \neq \mathbf{Y}_{j}$) from $\mathcal{V}\backslash\mathcal{N}(i)$ and let $\Delta \mathbf{A}_{i,j} =1$. 
Moreover, we use a hype-parameter $\gamma$ to control the ratio of noisy edges, i.e., $ |\mathcal{N}(i)^{\Delta}| = \gamma |\mathcal{N}(i)|$, where $\mathcal{N}(i)^{\Delta}$ represents the neighbour set of $v_i$ in $\Delta \mathbf{A}$.




\section{Experiments}
\label{sec:exp}
The experiments in this section aim to evaluate the answers to our two research questions. We first present the adverse effect on privacy risks of improving fairness, followed by demonstrating the outperformance of our PPFR method.

\subsection{Adverse effect on privacy risks of improving fairness}
\label{sec:exp:pre}
Introducing fairness regularisation (i.e., $f_{bias}$) into the loss function of a GNN model is effective in reducing bias \cite{KangHMT20}, while it potentially affects the privacy risk of edges (i.e., $f_{risk}$).
To verify the theoretical conclusion in Section \ref{sec:tradeoff}, this section assesses the effect of promoting individual fairness of nodes on the risk of edge leakage. 

\subsubsection{Preliminary Study Settings}  
The experimental settings in this section include datasets, models, and evaluation metrics.

\noindent
\textbf{Datasets and Models.}
In our experiments, we use Cora \cite{KipfW17}, Citeseer \cite{KipfW17}, and Pubmed \cite{KipfW17} datasets, which are commonly used in evaluating GNN in node classification. 
The models selected for this study are GCN \cite{KipfW17} with 16 hidden layers that employ ReLU and softmax activation functions. We use accuracy as the metric to evaluate the performance of GCNs.

\begin{table}
\centering
\caption{Comparison of the accuracy and bias of GCN models. 
In this table, ``vanilla''/``Reg'' indicates the GCN trained without/with fairness regularisation (i.e., $f_{bias}$).
}
\vspace{-5pt}
\label{tab:pre:performance_change}
\begin{tabular}{cccc}
\hline
Datasets                  & Methods & Acc$\uparrow$   & Bias$\downarrow$   \\ \hline
\multirow{2}{*}{Cora}     & Vanilla & 86.12 & 0.0766 \\ \cline{2-4} 
                          & Reg     & 85.38 & 0.0494 \\ \hline
\multirow{2}{*}{Citeseer} & Vanilla & 63.66 & 0.0445 \\ \cline{2-4} 
                          & Reg     & 63.11 & 0.0301 \\ \hline
\multirow{2}{*}{Pubmed}   & Vanilla & 85.37 & 0.0706 \\ \cline{2-4} 
                          & Reg     & 83.37 & 0.0108 \\ \hline
\end{tabular}
\vspace{-10pt}
\end{table}

\noindent
\textbf{Metrics.} The trustworthiness evaluation metrics in this paper includes the fairness and privacy aspects of GNNs.
\\
\noindent
\textbf{(1)}-\textit{fairness} ($f_{bias})$:
Following previous studies \cite{KangHMT20,ChangS21}, we combine the $Bias(\mathbf{Y},\mathbf{S})$ and original loss function together in the training phase to promote fairness. The similarity matrix $\mathbf{S}$ is defined as the Jaccard index and is derived from the graph structure \cite{KangHMT20}, and the bias is measured by $Bias(\mathbf{Y},\mathbf{S})$. The lower the bias value, the fairer the GNN prediction $\mathbf{Y}$. 
\\
\noindent
\textbf{(2)}-\textit{privacy} (AUC): In this paper, we assume that attackers can only query target models to obtain node predictions from target GNNs, which is the most practical link stealing attack (i.e., Attack-0 in \cite{HeJ0G021}). Based on prediction similarity, attackers attempt to infer the existence of an edge in any pair of nodes in the training graph. The edge leakage risk is measured by the AUC score, where larger values indicate higher privacy risks. Following a previous study \cite{HeJ0G021}, the prediction similarity is calculated using Cosine, Euclidean, Correlation, Chebyshev, Braycurtis, Canberra, Cityblock, and Sqeuclidean distances.


\begin{figure*}[ht!]
  \centering
  \includegraphics[width=\linewidth]{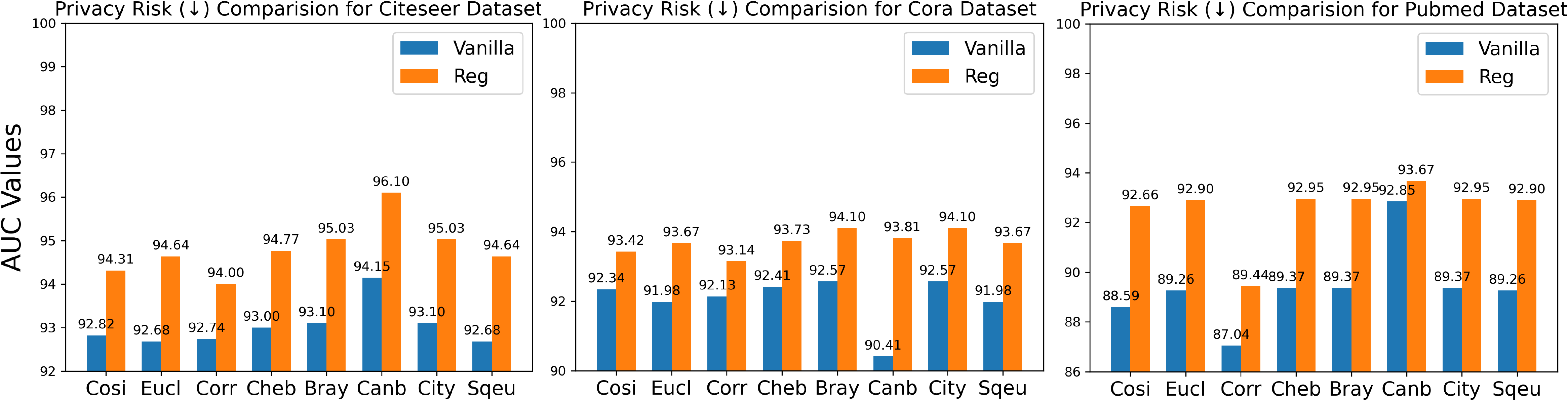}
  \caption{
  Privacy risks of edges before and after enhancing fairness. In this figure, ``vanilla''/``Reg'' indicates the GCN trained without/with fairness regularisation. The smaller AUC $\Leftrightarrow$ the better privacy.
  }
  \label{fig:pre:risk_change}
\end{figure*}

\subsubsection{Observations} Two main observations obtained from the evaluation results are summarised as follows.
\\
\noindent
\textbf{(1) Performance decrease when improving fairness:} 
In Table \ref{tab:pre:performance_change}, we observe that boosting fairness comes at the cost of model performance, i.e., the prediction accuracy is sacrificed when improving fairness on all datasets. 
\\
\noindent
\textbf{(2) Privacy risks increase when improving fairness:}
Performance reduction is not the only adverse effect. 
As shown in Fig. \ref{fig:pre:risk_change}, changes in AUC scores indicate that edge leakage risks increase when GNN fairness is promoted, which have been observed on all datasets.



\subsection{Performance of our PPFR method}


\subsubsection{Experimental Setup} After introducing additional GNN models and evaluation metrics, we will present related baseline methods to achieve our design goal, i.e., \textit{improving the individual fairness of nodes with limited GNN performance cost and restricted privacy risks of edges}.
\\
\noindent
\textbf{Models and Metrics.} Besides the GCN model, we also evaluated PPFR in the GAT \cite{VelickovicCCRLB18} and GraphSage \cite{HamiltonYL17} models. Fairness and privacy metrics follow that in Section \ref{sec:exp:pre}. Note that the privacy result in this section is the average AUC derived from eight different distances. 

In this paper, we use the following metric to evaluate the effect of a method $\Omega$ on both fairness and privacy, i.e.,
\begin{equation}
\label{eq:metric}
    \Delta = \frac{\Delta_{bias}\Delta_{risk}}{|\Delta_{acc}|},
\end{equation}
where $\Delta_{(\cdot)} = \frac{\Omega_{(\cdot)} - \mathtt{w/o}_{(\cdot)}}{\mathtt{w/o}_{(\cdot)}}$ takes the bias (i.e., $f_{bias}$), privacy risk (i.e., AUC), or performance (i.e., accuracy) values of methods $\Omega$ and $\mathtt{w/o}$ as inputs to evaluate the change ratio on different metrics when using method $\Omega$.
$\mathtt{w/o}$ represents the GNN model obtained by vanilla training. Note that the proposed metric $\Delta$ in (\ref{eq:metric}) satisfies:
(1) 
The numerical aspect of $\Delta$ measures the combined benefit of fairness and privacy per unit of performance cost incurred by a method on the GNN model.
(2)  A positive $\Delta$ (i.e., $\Delta > 0$) indicates that the method $\Omega$ can increase fairness and privacy simultaneously; otherwise, $\Delta$ is negative (i.e., $\Delta < 0$).
\\
\noindent
\textit{Remarks.} 
In our evaluations, $\Delta_{bias}$ and $\Delta_{risk}$ are not both positive simultaneously. 
Considering that the GNN utility generally decreases when improving fairness and privacy, we focus on evaluating whether a method could achieve positive $\Delta$ value with limited performance cost, i.e., small $|\Delta_{Acc}|$.

\begin{table*}[ht!]
\centering
\begin{threeparttable}
\caption{Effectiveness of our PPFR method. The smaller $\Delta_{bias} \Leftrightarrow$ the better fairness, the smaller $\Delta_{risk} \Leftrightarrow$ the better privacy, and the larger {\color[HTML]{00B050} positive} $\Delta$ $\Leftrightarrow$ the better balance of fairness and privacy.
}
\label{tab:effectiveness:ppfr}
\begin{tabular}{clrrrrrrrrr}
\hline
\multicolumn{1}{l}{}       &                             & \multicolumn{3}{c}{GCN}                                                             & \multicolumn{3}{c}{GAT}                                                             & \multicolumn{3}{c}{GraphSage}                                                        \\ \hline
Datasets                   & \multicolumn{1}{c}{Methods} & \multicolumn{1}{r}{$\Delta_{bias}\downarrow$} & \multicolumn{1}{r}{$\Delta_{risk}\downarrow$} & \multicolumn{1}{r}{$\Delta \uparrow$}     & \multicolumn{1}{r}{$\Delta_{bias}\downarrow$} & \multicolumn{1}{r}{$\Delta_{risk}\downarrow$} & \multicolumn{1}{r}{$\Delta \uparrow$}     & \multicolumn{1}{r}{$\Delta_{bias}\downarrow$} & \multicolumn{1}{r}{$\Delta_{risk}\downarrow$} & \multicolumn{1}{r}{$\Delta \uparrow$}      \\ \hline
                           & Reg                         & -35.51                 & 1.80                   & {\color[HTML]{FF0000} -0.744} & -51.94                 & 2.03                   & {\color[HTML]{FF0000} -0.964} & -99.99                 & -1.54                  & {\color[HTML]{00B050} 0.170}  \\
                           & DPReg                       & -74.54                 & -18.24                 & {\color[HTML]{00B050} 0.317}  & -68.89                 & -22.04                 & {\color[HTML]{00B050} 0.349}  & -99.94                 & -3.08                  & {\color[HTML]{00B050} 0.389}  \\
                           & DPFR                        & -1.17                  & 0.34                   & {\color[HTML]{FF0000} -0.007} & -6.68                  & 0.18                   & {\color[HTML]{FF0000} -0.005} & -99.93                 & -4.75                  & {\color[HTML]{00B050} 0.331}  \\
\multirow{-4}{*}{Cora}     & PPFR(Ours)                  & -11.75                 & -0.73                  & {\color[HTML]{00B050} 0.015}  & -25.63                 & -0.58                  & {\color[HTML]{00B050} 0.025}  & -99.92                 & -4.84                  & {\color[HTML]{00B050} 0.338}  \\ \hline
                           & Reg                         & -32.36                 & 1.91                   & {\color[HTML]{FF0000} -0.717} & -38.02                 & 2.44                   & {\color[HTML]{FF0000} -1.786} & -92.48                 & -0.28                  & {\color[HTML]{00B050} 0.049}  \\
                           & DPReg                       & 38.65                  & -15.87                 & {\color[HTML]{FF0000} -0.221} & -20.99                 & -21.84                 & {\color[HTML]{00B050} 0.103}  & -100.00                & -10.24                 & {\color[HTML]{00B050} 0.401}  \\
                           & DPFR                        & 0.22                   & -0.08                  & {\color[HTML]{FF0000} -0.001} & -3.21                  & -0.09                  & {\color[HTML]{00B050} 0.002}  & -97.89                 & -4.37                  & {\color[HTML]{00B050} 0.513}  \\
\multirow{-4}{*}{Citeseer} & PPFR(Ours)                  & -14.16                 & -0.31                  & {\color[HTML]{00B050} 0.009}  & -32.84                 & 0.36                   & {\color[HTML]{FF0000} -0.013} & -98.04                 & -4.51                  & {\color[HTML]{00B050} 0.530}  \\ \hline
                           & Reg                         & -84.70                 & 3.54                   & {\color[HTML]{FF0000} -1.280} & -61.54                 & 1.54                   & {\color[HTML]{FF0000} -1.476} & -99.59                 & 3.42                   & {\color[HTML]{FF0000} -0.951} \\
                           & DPReg                       & -45.89                 & -33.35                 & {\color[HTML]{00B050} 0.531}  & -16.78                 & -25.82                 & {\color[HTML]{00B050} 0.415}  & -99.38                 & -0.60                  & {\color[HTML]{00B050} 0.161}  \\
                           & DPFR                        & -29.60                 & 0.94                   & {\color[HTML]{FF0000} -0.198} & -57.34                 & 1.37                   & {\color[HTML]{FF0000} -2.061} & -41.44                 & -4.62                  & {\color[HTML]{00B050} 0.277}  \\
\multirow{-4}{*}{Pubmed}   & PPFR(Ours)                  & -31.30                 & -0.18                  & {\color[HTML]{00B050} 0.012}  & -10.49                 & -0.26                  & {\color[HTML]{00B050} 0.019}  & -41.63                 & -4.80                  & {\color[HTML]{00B050} 0.284}  \\ \hline
\end{tabular}
\begin{tablenotes}
    \footnotesize \item[*] In this table, columns $\Delta_{bias}$ and $\Delta_{risk}$ show evaluation results in the percentage (i.e.,\%) form.  
\end{tablenotes}
\end{threeparttable}
\end{table*}

\begin{figure*}[ht!]
  \centering
  \includegraphics[width=0.95\linewidth]{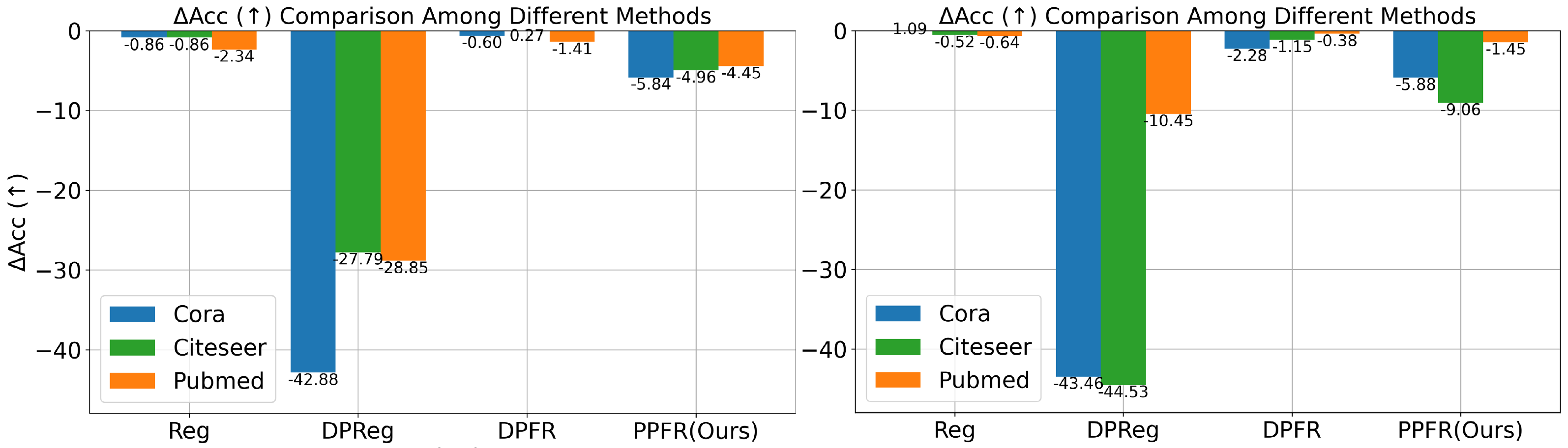}
  \caption{The accuracy cost of different methods for fairness on the GCN (left figure) and GAT (right figure) models. 
  In this figure, the higher $\Delta$Acc ($\%$) $\Leftrightarrow$ the better the performance of GNNs.
  }
  \label{fig:acc_cost}
  \vspace{-10pt}
\end{figure*}

\noindent
\textbf{Baselines.}
To evaluate our method (i.e., \textbf{PPFR}), we combine the GNN privacy and fairness methods together as baselines. 

In this paper, we consider differential privacy (DP) (i.e., EdgeRand and LapGraph \cite{0011L0022}) as the method to improve the edge privacy of GNNs, and both EdgeRand and LapGraph are used in generating a perturbed adjacency matrix. Specifically, to boost edge privacy, EdgeRand/LapGraph uses a randomisation/Laplacian mechanism to introduce edge noise into the original graph structure. 
We follow the parameter setting in \cite{0011L0022} to introduce $\epsilon$-edge DP (i.e., EdgeRand/LapGraph) into our evaluation. 
According to the previous study \cite{0011L0022}, EdgeRand and LapGraph have similar effectiveness when $\epsilon$ is small, while LapGraph is more applicable to large graphs \hlr{in terms of efficiency}. Thus, we use EdgeRand on Cora and Citeseer datasets and LapGraph on the Pubmed dataset.
Refer to Section \ref{sec:relatedwork:privacy} for why we focus on the edge DP method.

In this paper, \textbf{Reg} represents that fairness regularisation is introduced into the loss function of GNN models during their vanilla training. \textbf{DPReg} represents using the edge DP method and adding the fairness regularisation to loss simultaneously. \textbf{DPFR} indicates the combination of edge DP and our fairness-aware reweighting (FR) in this paper, where a perturbed graph is engaged in the fine-tuning phase.
In PPFR, the epoch number of the fine-tuning phase $\mathtt{e}_{re}$ is defined as $\mathtt{e}_{re}= s \mathtt{e}_{va}$, where $\mathtt{e}_{va}$ is the epoch number of vanilla training and $s$ is a hyper-parameter. In this paper, we set $s \in [0.1,0.25]$ for different combinations of model and datasets and $\alpha = 0.9$, $\beta = 0.1$ for QCLP. 


\subsubsection{Evaluation results}
This section evaluates the effectiveness of the PPFR method in increasing fairness with restricted privacy risks (i.e., Table \ref{tab:effectiveness:ppfr}) and its performance cost (i.e., Figures \ref{fig:acc_cost} and \ref{fig:acc_cost2}) to demonstrate its performance.
Note that, (1) desired methods are those with $\Delta acc (\uparrow)$, $\Delta_{bias} (\downarrow)$, and $\Delta_{risk} (\downarrow)$; (2) the weight ($\omega$) relations of different metrics in  comprehensively evaluating different methods are shown by the following ranking: 
\begin{equation*}
    \omega (\Delta acc) > \omega (\Delta_{bias}) = \omega (\Delta_{risk}),
\end{equation*}
as the performance of GNNs lies in their central position to serve users, even if when improving fairness or decreasing privacy risks.


Evaluation results demonstrate that PPFR outperforms baselines in \textit{balancing fairness and privacy} and \textit{limiting performance cost}. 
Specifically, our main observations on PPFR are:

\noindent
\textbf{Boosting fairness with restricted privacy risks. } The outperformance of our PPFR method can be summarised as follows.
\\
\noindent
(1) \ul{PPFR owns positive $\Delta$ values or better behaviour:} Among all methods, only PPFR and DPReg have a positive sign in almost all cases, indicating that they can increase fairness while restricting privacy risks. 

Although the DPReg method can achieve positive $\Delta$ results, several weaknesses have been recognised in this method.
First, edge DP is applied to the training graph of GNNs to protect the privacy of edges, however, introducing these random edge noises potentially into the vanilla training process harms the performance of GNNs. 
For example, we observe a huge performance cost (e.g., $\Delta acc= -44.53\%$ in (Citeseer,GAT) case) according to Fig. \ref{fig:acc_cost}.
Second, the sampling operation in GraphSage \cite{HamiltonYL17} extremely reduces the effectiveness of edge DP. 
This is because only limited noisy edges from edge DP can participate in the training process of GNNs with the sampling operations. 
Our evaluation results in Table \ref{tab:effectiveness:ppfr} show that the $\Delta_{risk}$ of DPReg with the GraphSage (e.g., $-0.60\%$ on Pubmed) is obvious less than that with GCN and GAT (e.g., $-33.35\%$ and $-25.82\%$ on Pubmed, respectively).
In contrast, our PPFR method uses a specially designed graph structure modification strategy and employs this privacy-aware strategy in the fine-tuning phase, which helps to protect privacy of edges and reduce performance cost at the same time.

\begin{figure*}[ht!]
  \centering
  \includegraphics[width=\linewidth]{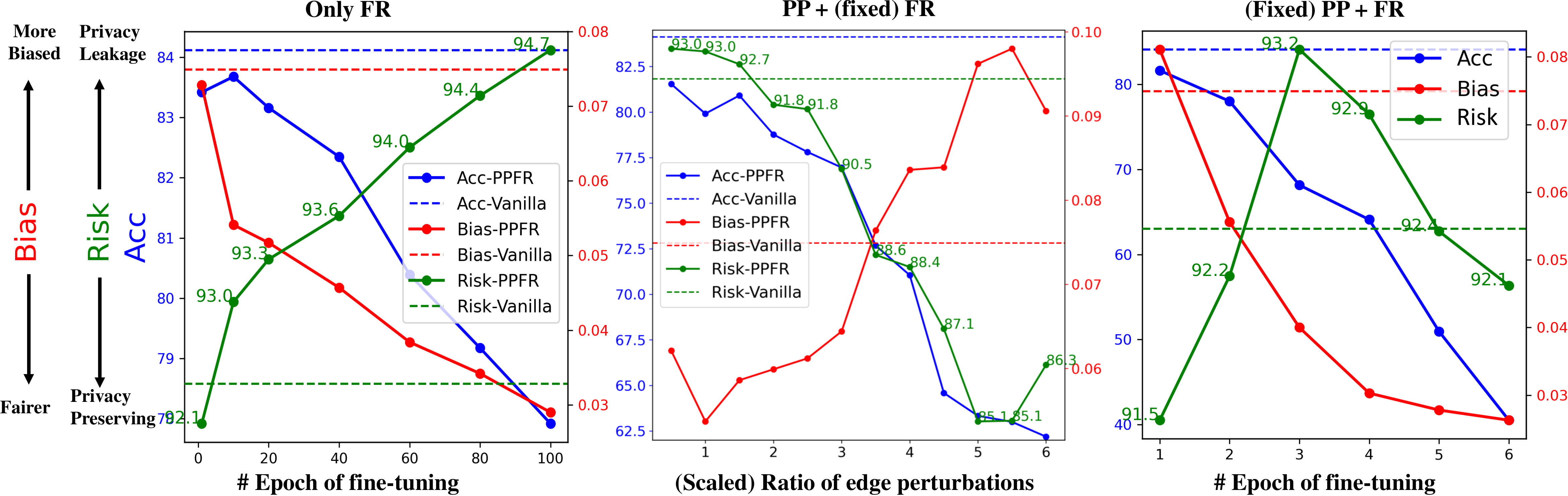}
  \caption{Changes of performance, fairness, and privacy of GNNs when changing parameters in PPFR. In this figure, the higher Acc ($\uparrow$) $\Leftrightarrow$ the better performance, the lower Risk ($\downarrow$) $\Leftrightarrow$ the better privacy, and the lower Bias ($\downarrow$) $\Leftrightarrow$ the better fairness. On the left, we use zero edge perturbations in PP; in the middle, we fix the $\#$ epoch in the FR; on the right, e use fixed edge perturbations in PP, and change the $\#$ epoch in the FR. The dashed lines indicate (accuracy, fairness, and privacy) evaluation results from the vanilla GNN models.}
  \label{fig:ablation}
  \vspace{-10pt}
\end{figure*}

Moreover, PPFR and DPReg behave differently in cases where they have negative values. 
For example, in case (Citeseer, GCN), DPReg focusses on reducing privacy risks while leading to bias increase, which indicates that noisy edges from DP potentially harm individual fairness. 
In contrast, in the case (Citeseer, GAT), PPFR produces the same level of debiasing and restrains the increased privacy risks in the Reg method.
\\
\noindent
(2) \ul{Introducing privacy consideration helps reduce leakage risks:} Compared to the Reg method, methods (i.e., DPReg, DPFR, and PPFR) involving privacy consideration have a larger $\Delta$, indicating that introducing privacy consideration in GNN helps reduce edge leakage risks.
\\
\noindent
(3) \ul{PP outperforms DP when combined with FR:} According to the $\Delta$ results in Table \ref{tab:effectiveness:ppfr}, our PPFR method performs better in increasing fairness with restricted privacy risks than DPFR. This is because, when introducing same-level noises, the randomisation/Laplacian edge noises in the DP are less effective than the heterophilic edge noise in our privacy-aware perturbations (PP) method, which is specially designed based on our theoretical analysis in Section \ref{sec:method:ppfr:pp}.


\noindent
\textbf{Performance cost comparison.} Current works have shown that both the fairness \cite{DongKTL21} and privacy \cite{0011L0022} of GNNs are at the cost of performance. However, the performance of GNNs lies at their core when providing high-quality services to users. 
As shown in Figures \ref{fig:acc_cost} and \ref{fig:acc_cost2}, we observe that:
\\
\noindent
(1) \ul{PPFR has a lower performance cost than DPReg:} As a straightforward baseline, directly combining the privacy (DP) and fairness (Reg) methods together (i.e., DPReg) is at a huge performance cost. 
For example, the performance of DPReg in cases (Cora, GCN/GAT; Citeseer, GAT) drops by more than $40\%$ compared with the vanilla GNN models, potentially leading to the impracticability of GNN systems. 
This is because involving the same level (i.e., same number) of noisy edges in the initial training of GNNs harms performance more than involving them in the latter training phase. 
In contrast, due to considering GNN performance in our FR and the two-phase training design, our PPFR method has a lower performance cost.
\\
\noindent
(2) \ul{PPFR has the limited performance cost compared to Reg:} Compared to the Reg method, our PPFR method only leads to a limited extra performance cost. 
However, this additional performance is necessary because protecting the privacy of GNNs comes at the cost of performance \cite{0011L0022}. 
Overall, our PPFR method still maintains the competent performance (e.g., $81.57\%$ accuracy on Pubmed, GCN) when serving GNN users.

\subsection{Ablation Study}



In this section, we conduct ablation studies of parameters in the PP (i.e., ratio of edge perturbations) and FR (i.e., $\#$ epoch in fine-tuning). These ablation studies are evaluated on the Cora dataset with GAT model. 
We observe that:
\\
\noindent
(1) \ul{Only improving fairness harms the privacy of edges}: As shown in (left) Fig. \ref{fig:ablation}, both accuracy and privacy risks suffer the adverse effect of improving individual fairness, which further answers RQ1 and confirms our observation in Section \ref{sec:exp:pre}.  
\\
\noindent
(2) \ul{PP is at the cost of GNN performance}: In (middle) Fig. \ref{fig:ablation}, GNN performance continues to decrease when more privacy-aware perturbations occur, which is consistent with previous studies on GNN privacy \cite{0011L0022}.
\\
\noindent
(3) \ul{Trade-off between fairness and privacy}: Besides our previous observation that improving fairness harms the privacy of edges, (middle) Fig. \ref{fig:ablation} shows that reducing the risk of privacy makes GNNs more biased, which potentially indicates the existence of a trade-off between individual fairness of nodes and privacy risks of edges.
This observation confirms our conclusion in Fig. \ref{fig:interactions}, that is, there is a trade-off between any two of performance, individual fairness of nodes, and privacy risks of edges.


We also evaluate our method in the case of changing the epoch number in FR while using non-zero fixed edge perturbations in (right) Fig. \ref{fig:ablation}, we observe that
\\
\noindent
(4) \ul{PP helps restraining privacy risks}: Compared to (left) Fig. \ref{fig:ablation}, privacy risks are restricted at the same level as the vanilla GNNs when the PP method is introduced. 
This evaluation verifies that our privacy-aware perturbation strategy is effective in restricting the increased privacy risks of edges when improving individual fairness of nodes.
\\
\noindent
(5) \ul{Considering multi aspects of trustworthy GNNs bring more performance cost}: GNN performance in (right) Fig. \ref{fig:ablation} drops faster than that in (left) Fig. \ref{fig:ablation}, indicating that considering both fairness and privacy leads to a higher performance cost.

\begin{table}[ht!]
\centering
\caption{
\hlr{
Evaluations of the GCN on datasets with weak homophily.
}
}
\label{tab:effectiveness:ppfr2}
\begin{tabular}{clrrrr}
\hline
Datasets                  & \multicolumn{1}{c}{Methods} & \multicolumn{1}{c}{$\Delta_{acc}$} & {$\Delta_{bias}\downarrow$} & \multicolumn{1}{r}{$\Delta_{risk}\downarrow$} & \multicolumn{1}{r}{$\Delta \uparrow$}       \\ \hline
                          & Reg                           & -3.64                   & -37.33                   & 2.21                    & {\color[HTML]{FE0000} -0.227}   \\
                          & DPReg                         & -36.79                  & -58.06                   & -34.08                  & {\color[HTML]{00B050} 0.538}    \\
                          & DPFR                        & -2.65                   & -42.86                   & -0.21                   & {\color[HTML]{00B050} 0.033}    \\
\multirow{-4}{*}{Enzymes} & PPFR(Ours)                  & -4.35                   & -32.26                   & -0.17                   & {\color[HTML]{00B050} 0.012}    \\ \hline
                          & Reg                           & -6.55                   & -12.61                   & -15.62                  & {\color[HTML]{00B050} 0.301}    \\
                          & DPReg                         & -3.12                   & 2775.68                  & -37.36                  & {\color[HTML]{FE0000} -332.668} \\
                          & DPFR                        & 2.40                    & -77.48                   & -7.20                   & {\color[HTML]{00B050} 2.320}    \\
\multirow{-4}{*}{Credit}  & PPFR(Ours)                  & 3.26                    & -91.89                   & -9.42                   & {\color[HTML]{00B050} 2.652}    \\ \hline
\end{tabular}
\begin{threeparttable}
\begin{tablenotes}
    \footnotesize \item[*] Columns $\Delta_{acc}$, $\Delta_{bias}$ and $\Delta_{risk}$ show evaluation results in the percentage (i.e.,\%) form. 
    In cases where the "Reg" has a positive $\Delta$ value (e.g., results with Credit), although using our methodology is not obligatory, our approach excels in concurrently augmenting fairness and privacy (i.e., positive $\Delta$), while also offering an incremental increase in accuracy.
\end{tablenotes}
\end{threeparttable}
\end{table}

\begin{figure}[h!]
  \centering
  \includegraphics[width=0.95\linewidth]{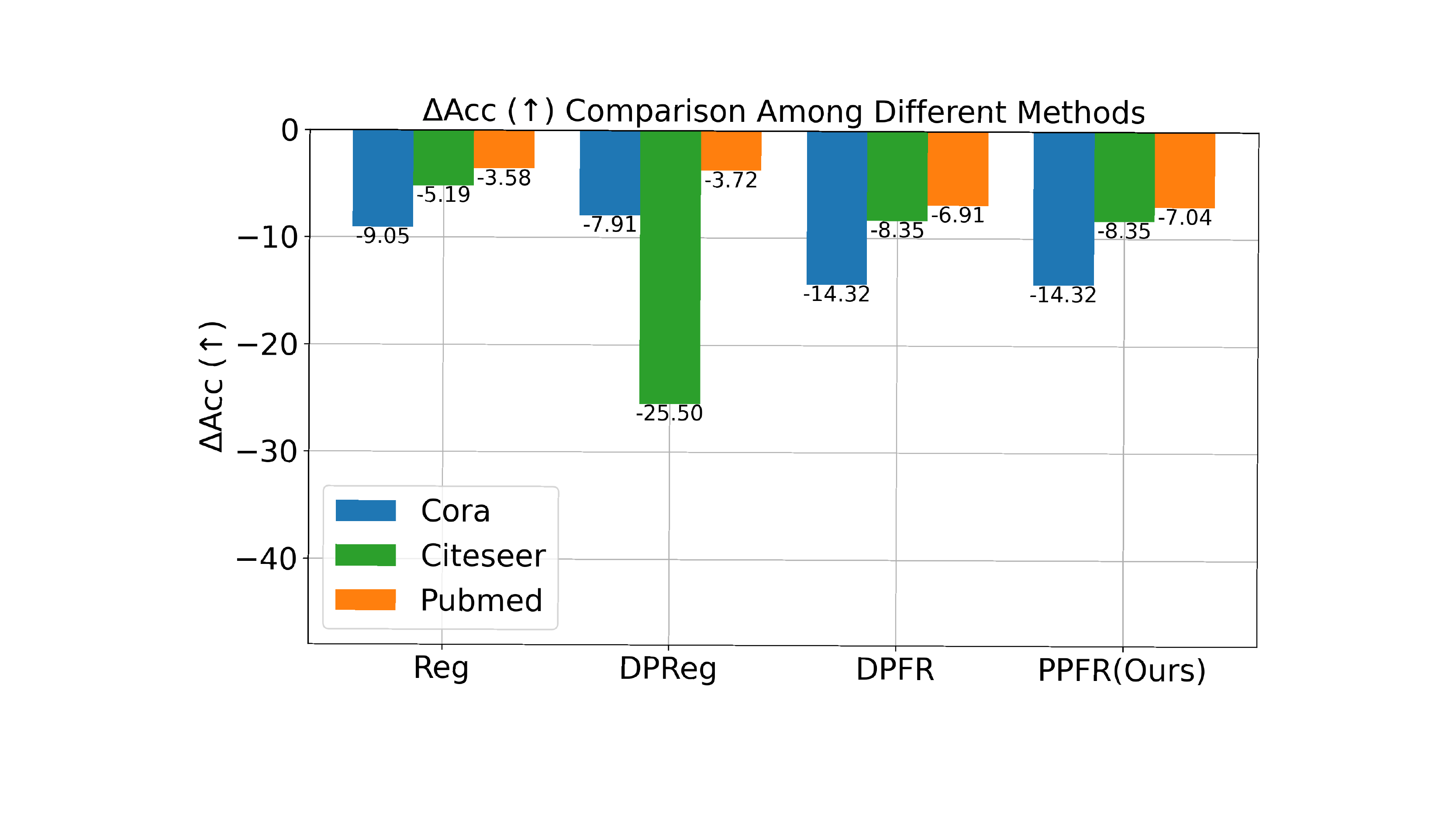}
  \caption{The accuracy cost of different methods on the GraphSAGE models.}
  \label{fig:acc_cost2}
  \vspace{-15pt}
\end{figure}

\subsection{Discussion}


\noindent
\hlr{
\textbf{Performance on graphs with weak homophily.}
The above results (i.e., Table \ref{tab:effectiveness:ppfr}) confirm the existence of trade-off between fairness and privacy of GNNs on graph datasets with high homophily (e.g., the homophily value of Cora, Citeseer and Pubmed is 0.81, 0.74, and 0.80, respectively). 
To further investigate the trade-off and performance of different methods, we performed evaluations on graphs with weak homophily.
As shown in Table \ref{tab:effectiveness:ppfr2}, we observe that:
\\
\noindent
(1) \ul{Due to the complex interaction between fairness and privacy, their trade-off could be limited or non-existent when handling graphs with weak homophily}: 
For example, for graphs with weak homophily (e.g., 0.66/0.62 for Enzymes \cite{dobson2003distinguishing}/Credit\cite{AgarwalLZ21}), the $\Delta$ values (e.g., -0.227/0.301) of the Reg fairness method in these datasets are obviously higher than those in Table \ref{tab:effectiveness:ppfr} (e.g., -0.744 on Cora).
Although our proposition \ref{eq:proposition} provides insight in understanding the trade-off, it is necessary to involve more factors to comprehensively investigate the complex interaction between fairness and privacy.
\\
\noindent
(2) \ul{When combined with FR, PP and DP have the same level of performance on graphs with weak homophily}:
Unlike consistently achieving limited restrictions on privacy risks in Table \ref{tab:effectiveness:ppfr}, the DP (i.e., LapGraph) method can obtain results comparable to our PP method.
One potential reason is that the noisy edges from DP are more similar to the edge distribution of weak homophily graphs than that of high homophily graphs, making it harder for current privacy attack methods to distinguish real edges and fake noisy edges. 
}

\noindent
\hlr{
\textbf{The scope of fairness-privacy trade-off}}.
It is worthy to point out the conclusion in this paper, i.e., \textit{there is a trade-off between fairness and privacy of GNNs}, is derived and verified by considering the individual fairness of nodes and privacy risks of edges, where the similarity in the individual fairness is defined as the Jaccard similarity. 
As we discussed before, this study is motivated by the potential connection between the fairness goal and the intuition of link stealing attacks, that is \textit{``decreasing the prediction difference between similar nodes w.r.t. their Jaccard similarity''} facilitates distinguishing connected node pairs and unconnected node pairs by considering that \textit{``nodes with more similar predictions are more likely to be connected''}.

Note that trustworthiness-driven studies have emerged to improve the fairness or privacy of GNNs by utilising various methods or design goals (e.g., individual fairness with another perspective of similarity) \cite{abs-2205-07424}.
\hlr{
Our conclusion may not stand in the scenario where the graph properties and the definition (e.g., fairness) or method for trustworthiness are different from that in this paper.
This is due to the way fairness interacts with privacy impacted by various factors, such as the dataset property \hlr{(e.g., weak homophily in Table \ref{tab:effectiveness:ppfr2})}, GNN architectures, fairness/privacy goal, evaluation metrics, etc.
To comprehensively build trustworthy graph learning systems, it is crucial to investigate the interaction between fairness and privacy in a systematic manner.}
\hlr{
\\
\noindent
\textit{Remarks}. Our re-weighting may not be optimal due to factors like the limited number of training data. Additionally, the perturbation approach is obtained by focusing on intra-class node pairs. Therefore, future comprehensive analyses of fairness and privacy could lead to improved methods.
}

\section{Conclusion}
\label{sec:con}
In this paper, we investigate the interaction between the fairness and privacy of GNNs, which is indispensable to comprehensively build trustworthy GNNs. 
We theoretically analyse and empirically observe the adverse effects of node fairness on privacy risks of edges, and devise a fine-tuning method named PPFR to increase GNN fairness with limited performance cost and restricted privacy risk, whose effectiveness is validated by our experimental evaluations on real-world datasets. 
Our work demonstrates the feasibility of considering performance, fairness, and privacy simultaneously in GNNs.
Experimental evaluations demonstrated the trade-off between the fairness and privacy of GNNs, and the outperformance of our PPFR method in navigating performance, fairness, and privacy.

In the future, we will explore the interactions among other different trustworthiness aspects (e.g. privacy \cite{DuanYH22}, explainability \cite{GaoWZCMWL23,RorsethGGKSS23}) to build trustworthy GNNs in a holistic way.
\hlr{For graphs with weak homophily, we will also investigate how to automatically determine whether it is necessary to employ methods that consider both fairness and privacy simultaneously.}
Another promising direction is to study how to automatically and dynamically change graph data and reweight the loss function at the same time during the training phase to achieve competent and trustworthy GNNs.

\section*{Acknowledgments}
This research was supported in part by the Australian Research Council (ARC) under grants FT210100097 and DP240101547. Xingliang Yuan and Shirui Pan are also supported by CSIRO – National Science Foundation (US) AI Research Collaboration Program.


\balance
\bibliographystyle{IEEEtran}
\bibliography{references}


\end{document}